\documentclass[twoside]{article}

\usepackage[accepted]{aistats2021}
\usepackage{graphicx}
\usepackage{subcaption}
\usepackage{amsmath}
\usepackage{amssymb}
\usepackage{amsthm}
\usepackage{makecell}
\usepackage{xcolor}
\usepackage{bbm}
\usepackage{algcompatible,lipsum}

\usepackage{mathtools}

\usepackage{enumitem,kantlipsum}

\newtheorem{theorem}{Theorem}
\newtheorem{lemma}{Lemma}

\allowdisplaybreaks

\usepackage{hyperref}
\usepackage{natbib}
\setcitestyle{authoryear,open={(},close={)}}

\begin{document}

\runningauthor{Valerii Likhosherstov*, Jared Davis*, Krzysztof Choromanski, Adrian Weller}

%

%

\twocolumn[

\aistatstitle{CWY Parametrization: a Solution for Parallelized Optimization of Orthogonal and Stiefel Matrices}

\aistatsauthor{ Valerii Likhosherstov$^{*1}$ \And Jared Davis$^{*2,3}$ \And  Krzysztof Choromanski$^{4,5}$ \And Adrian Weller$^{1,6}$ }

\aistatsaddress{ $^1$University of Cambridge \quad $^2$DeepMind \quad $^3$Stanford University \quad $^4$Google Brain \quad $^5$Columbia University \\ $^6$The Alan Turing Institute  \quad $^*$Equal contribution} ]

\begin{abstract}
We introduce an efficient approach for optimization over orthogonal groups on highly parallel computation units such as GPUs or TPUs. As in earlier work, we parametrize an orthogonal matrix as a product of Householder reflections. However, to overcome low parallelization capabilities of computing Householder reflections sequentially, we propose employing an accumulation scheme called the compact WY (or CWY) transform -- a compact parallelization-friendly matrix representation for the series of Householder reflections. We further develop a novel Truncated CWY (or T-CWY) approach for Stiefel manifold parametrization which has a competitive complexity and, again, yields benefits when computed on GPUs and TPUs. We prove that our CWY and T-CWY methods lead to convergence to a stationary point of the training objective when coupled with stochastic gradient descent. We apply our methods to train recurrent neural network architectures in the tasks of neural machine translation and video prediction.
\end{abstract}
\section{INTRODUCTION}

Training weight matrices in a neural network with an orthogonality constraint gives various benefits for a deep learning practitioner, including enabling control over the norm of the hidden representation and its gradient which can be helpful for several reasons. A series of works addresses the problems of exploding or vanishing gradients in recurrent neural networks (RNNs) \citep{hochreiter1998vanishing} by using orthogonal or unitary transition matrices \citep{arjovsky,wisdom,eunn,hh,helfrich,cheap}. Further, orthogonality appears to improve forward and backward information propagation in deep convolutional neural networks where convolutions are parametrized by a Stiefel manifold---a general class of orthogonal matrices \citep{huang,cnn,rgdst}. The norm-preserving property of an orthogonal linear operator helps to gain control over the Lipschitz constant of the deep architecture and, therefore, can enhance adversarial robustness of the model and its generalization capabilities both in theory and practice \citep{parseval}. Orthogonality is also useful when designing invertible constructions for flow-based generative modelling \citep{snf}.

Yet there is a lack of an orthogonal optimization method which is compatible with the industry-standard use of highly-parallel devices (GPU or TPU) for computations. Indeed, existing approaches for training an $N \times N$ orthogonal matrix can be grouped into two categories (see Table \ref{table:1}):
\begin{itemize}
    \item Algorithms involving expensive operation of $N \times N$-sized matrix inversion or exponent \citep{wisdom,cheap,helfrich} resulting in at least $O(N^2 \log N)$ parallel complexity \citep{par}.
    \item Algorithms decomposing the orthogonal operator into a set of $L < N$ linear operators applied sequentially \citep{eunn,hh}, not taking full advantage of parallel matrix multiplication on GPU and TPU \citep{parmatmul,par}, and resulting in at least $O (L)$ parallel complexity.
\end{itemize}
Hence, there is a critical gap, with no method which works when a) cubic time is prohibitive and b) large $L$ for non-cubic approaches is slow while small $L$ seriously restricts model capacity.

We present a new approach to optimization over orthogonal matrices, focusing on computational efficiency. We employ the compact WY (or CWY) transform, a scheme for the composition of several Householder reflections \citep{householder}. Our proposed approach has several advantages:

1. While in exact arithmetic being equivalent to decomposition into Householder reflections \citep{hh}, the parallel complexity of the algorithm \textbf{is only $\boldsymbol{O(\log (L N))}$} with $O(L^2 \log L)$ preprocessing (see Table \ref{table:1}) which makes it especially efficient when executed on GPU or TPU. We observe $20 \times$ speedup in practice compared to sequential Householder reflections \citep{hh} (see Table \ref{fig:hr-cwy-speed}) and \textbf{1-3 orders of magnitude speedups compared to matrix exponential and Cayley map} (Figure \ref{fig:timecomp}).



2. We introduce an extension for parametrizing Stiefel manifolds -- nonsquare generalizations of orthogonal matrices. The extension scheme, named ``Truncated CWY'' (or T-CWY), is to our knowledge \textbf{a novel parametrization of the Stiefel manifold which requires the smallest number of floating point operations (FLOPs)} among methods for Stiefel optimization (see Table \ref{table:2}).

3. Finally, we prove that SGD based on CWY or T-CWY leads to a gradient norm convergence to zero with $o(K^{-0.5 + \epsilon})$ rate for any $\epsilon > 0$ where $K$ is an iteration index.

We evaluate CWY on standard benchmarks (Copying task, Pixel-by-pixel MNIST) and neural machine translation. We evaluate T-CWY on the task of video prediction.
All theoretical results are proven in Appendix \ref{sec:proofs}.

\section{RELATED WORK} \label{sec:rwork}

We discuss orthogonality in the motivating example of RNN gradient explosion and vanishing. Then we review orthogonal optimization methods and their properties, summarized in Tables \ref{table:1} and \ref{table:2}.

\subsection{Gradient Explosion and Vanishing}

The rollout of a recurrent neural network (RNN) can be formalized as a series of computations \citep{jordan1990attractor}:
\begin{equation}
    y_t := W h_{t - 1} + b; \quad h_t := \sigma ( y_t + V x_t ); \label{eq:rnn}
\end{equation}
for $t = 1, \dots, T$. Here $x_1, \dots, x_T \in \mathbb{R}^K$ are the states of an observed sequence $X = \{ x_1, \dots, x_T \}$ from the training set, $h_0, \dots, h_T \in \mathbb{R}^N$ is a sequence of hidden states ($h_0$ is fixed and usually zero), $W \in \mathbb{R}^{N \times N}$ is a transition matrix, $b \in \mathbb{R}^N$ is a bias term, $V \in \mathbb{R}^{N \times K}$ is an input transformation matrix and $\sigma (\cdot)$ is an elementwise nonlinear function. $N$ and $K$ are the dimensions of the hidden and observed states respectively. In this work, we are interested in constraining $W$ to a restricted (orthogonal) form $Q$, which we shall make precise shortly. Let $C$ denote an objective function to minimize. For ease of illustration, we assume that $C$ is a function of the last hidden state: $C = C(h_T)$. Then one has the following expression for gradients w. r. t. intermediate hidden states:
\begin{equation*}
    \frac{\partial C}{\partial h_t} = \biggl( \prod_{k = t}^{T - 1} \frac{\partial h_{k + 1}}{\partial h_k} \biggr) \frac{\partial C}{\partial h_T} = \biggl( \prod_{k = t}^{T - 1} J_\sigma (h_k) W^\top \biggr) \frac{\partial C}{\partial h_T} ,
\end{equation*}
where $J_\sigma$ is the Jacobian of $\sigma(\cdot)$ applied elementwise. In practice, the expression leads to the hidden state norm increasing exponentially fast with $T - t$ when $\| W \|_2 = \sup_{\| h \|_2 = 1} \| W h \|_2 > 1$ (\textit{gradient explosion}) or decreasing exponentially fast when $\| W \|_2 < 1$ (\textit{gradient vanishing}). Both effects are undesirable as they lead to unstable learning and inability to capture long-term dependencies in the data.
To alleviate this problem, \citet{arjovsky} proposed using an orthogonal or unitary matrix $W$, that is to set either $W = Q \in \mathcal{O}(N)$ or $W = Q \in \mathcal{U}(N)$. Here $\mathcal{O} (N) = \{ Q \in \mathbb{R}^{N \times N} \, | \, Q^\top Q = I \}$ is called the \textit{orthogonal group}, $\mathcal{U}(N) = \{ Q \in \mathbb{C}^{N \times N} \, | \, Q^H Q = I \}$ is called the \textit{unitary group}, $Q^H$ denotes the conjugate transpose and $I$ denotes an identity matrix, with shape inferred from the context. Since orthogonal or unitary linear operators are $l_2$-norm preserving (i.e. $\forall h: \| Q h \|_2 = \| h \|_2$), the norm of the intermediate state gradient is approximately constant when $J_\sigma (h_k) \approx I$. Next we discuss approaches to tackle the constrained optimization problem formulated as
\begin{equation} \label{eq:oc}
    \min_{W, V, b} C \quad \text{s.t. } W=Q \in \mathcal{O} (N) \quad (\text{or } Q \in \mathcal{U} (N) ) .
\end{equation}

\subsection{Orthogonal Optimization}

We review two families of earlier methods to solve the constrained optimization problem (\ref{eq:oc}).

\subsubsection{Parametrization}

This is a family of methods constructing $Q$ as a function of unconstrained parameters, on which standard gradient descent can be performed.

\textbf{URNN} (Unitary Recurrent Neural Network, \citealp{arjovsky}) expresses Q as $D^{(3)} H^{(2)} F^{-1} D^{(2)} \Pi H^{(1)} F D^{(1)}$, where $D^{(1)}, D^{(2)}, D^{(3)}$ are parametrized diagonal unitary matrices, $H^{(1)}, H^{(2)}$ are parametrized Householder reflections (\citep{householder}, see the definition below), $F$ is a discrete Fourier transform matrix and $\Pi$ is a random permutation matrix.

\textbf{EURNN} (Efficient Unitary RNN, \citealp{eunn}) parametrizes $Q = D F^{(1)} F^{(2)} \dots F^{(L)} \in \mathcal{U} (N)$ where $L \leq N$, $D$ is diagonal unitary and $F^{(i)} \in \mathbb{C}^{N \times N}$ are permuted block-diagonal with $2 \times 2$ blocks.

\textbf{HR} (Householder reflections, \citealp{hh}) decomposes $Q = H(v^{(1)}) \dots H(v^{(L)}) \in \mathcal{O}(N)$ where for each nonzero $v \in \mathbb{R}^N$,  $H (v) = I - 2  v v^\top / \| v \|_2^2 \in \mathcal{O}(N)$ is a \textit{Householder reflection}.

\textbf{EXPRNN} (Exponent RNN, \citealp{cheap}). 
This method takes advantage of the fact that the matrix exponent $\exp (A)$ 
is a surjective mapping from the set of skew-symmetric matrices $\text{Skew}(N) = \{ A \in \mathbb{R}^{N \times N} \, | \, A = -A^\top \}$ to the \textit{special orthogonal group} $\mathcal{O}^{+1}(N)$, where for $s = \pm 1$ we define $\mathcal{O}^s (N) = \{ Q \in \mathcal{O}(N) \, | \, \det Q = s \}$. Notice that $\mathcal{O}(N) = \mathcal{O}^{+1}(N) \cup \mathcal{O}^{-1}(N)$.

\textbf{SCORNN} (Skew Cayley, \citealp{helfrich}) uses the \textit{Cayley transform} instead of matrix exponent: $Q = \text{Cayley} (A) = (I + A/2)^{-1} (I - A/2)$ which is a bijective map from $\text{Skew}(N)$ to $\mathcal{O}^{+1}(N) \setminus \Theta$ where $\Theta$ is a set of matrices with $-1$ eigenvalue. To cover all matrices from $\mathcal{O} (N)$, $Q$ is scaled as $\tilde{Q} = Q \tilde{D}$ where $\tilde{D}$ is a diagonal matrix with $\pm 1$ values. The number of $-1$'s in $\tilde{D}$ is a hyperparameter, which requires an additional search method. For fair comparison, we fix $\tilde{D} = I$.

\textbf{OWN} (Orthogonal Weight Normalization, \citealp{huang}). 
This method considers the more general task of optimizing a function over the \textit{Stiefel manifold} $\text{St} (N, M) = \{ \Omega \in \mathbb{R}^{N \times M} \, | \, \Omega^\top \Omega = I \}$ where $M \leq N$,  which generalizes the set $\mathcal{O}(N)$.
$\Omega$ 
is set as $\Omega = \tilde{V} P \Lambda^{-1/2} P^\top$, $\tilde{V} = (V - \frac{1}{N} \mathbf{1} \mathbf{1}^\top V)$ where $P \Lambda P^\top$ is an eigendecomposition of matrix $\tilde{V}^\top \tilde{V} \in \mathbb{R}^{M \times M}$ and $\mathbf{1}$ is the all-ones $N$-vector.

\subsubsection{Riemannian Gradient Descent (RGD)}

These methods instead consider gradient descent directly on the Stiefel manifold. Rather than  ``straight-line'' steps as in typical gradient descent, RGD goes along a curve which a) lies in $\text{St} (N, M)$ and b) points in the direction of fastest descent along the manifold. More precisely, RGD starts with a predefined matrix $\Omega^{(0)} \in \text{St} (N, M)$ and makes sequential updates of the type $\Omega^{(k)} :=  g_k (\eta_k)$ where $\eta_k$ is a step size, $g_k: \mathbb{R} \to \text{St} (N, M)$, $g_k(0) = \Omega^{(k - 1)}$ and $g_k'(0)$ is the gradient $\frac{\partial f}{\partial \Omega} (\Omega^{(k - 1)})$ projected onto the \textit{tangent space} $\mathcal{T}_{\Omega^{(k - 1)}}$ -- a linear space approximating the Stiefel manifold $\text{St} (N, M)$ at the point $\Omega^{(k - 1)}$. It is known that $\mathcal{T}_\Omega = \{ Z \in \mathbb{R}^{N \times M} \, | \, Z^\top \Omega \in \text{Skew} (M) \}$.
For a rigorous introduction to Riemannian manifolds and Riemannian Gradient Descent see \citep{absil}.

In a Riemannian manifold, the tangent space $\mathcal{T}_\Omega$ must have an inner product, usually chosen as either the \textit{canonical inner product} $\langle Z_1, Z_2 \rangle_1 = \text{Tr} (Z_1^\top (I - \frac{1}{2} \Omega \Omega^\top) Z_2)$ or \textit{Euclidean inner product} $\langle Z_1, Z_2 \rangle_2 = \text{Tr} ( Z_1^\top Z_2 )$. Consequently, the projection of the gradient has the form: $g'_k (0) = A^{(k - 1)} \Omega^{(k - 1)}$, $\quad A^{(k - 1)} = \widehat{A}_i^{(k - 1)} - \widehat{A}_i^{(k - 1) \top}$ where $\widehat{A}_1^{(k - 1)} = \frac{\partial f}{\partial \Omega} (\Omega^{(k - 1)}) \Omega^{(k - 1)\top}$ corresponds to the canonical inner product choice, and $\widehat{A}_2^{(k - 1)} = \widehat{A}_1^{(k - 1)} - \frac{1}{2} \Omega^{(k - 1)} \Omega^{(k - 1)\top} \widehat{A}_1^{(k - 1)}$ corresponds to the Euclidean inner product choice. Next, there is freedom in choosing the type of $g_k (\eta)$ function. Two popular choices are 1) \textit{Cayley retraction} $g_k^\text{Cay} (\eta) = \text{Cayley} (\eta A^{(k - 1)}) \Omega^{(k - 1)}$ and 2) \textit{QR-decomposition retraction} $g_k^\text{QR} (\eta) = \mathrm{qf} (\eta A^{(k - 1)} \Omega^{(k - 1)})$ where $\mathrm{qf} (\cdot)$ denotes a Q matrix of the argument's QR decomposition so that diagonal elements of the R matrix are positive. \citet{wisdom,rgdst} evaluate performance of RGD in the context of deep learning.

\subsection{Runtime Complexity}

\begin{table*}[t]
\caption{Comparison of runtime complexity required for a forward pass through RNN. To report parallel complexity we use that a) a product of $d_1 \times d_2$ and $d_2 \times d_3$-sized matrix takes $O(\log (d_1 d_2 d_3))$ time (distribution over $O (d_1 d_2 d_3)$ processes) \citep{parmatmul} and b) finding an inverse of $d_1 \times d_1$-sized matrix takes $d_1^2 \log d_1$ time (distribution over $O(d_1)$ processes) \citep{par}. All complexities are in $O(\cdot)$ notation, terms related to $V x_t$ computation are omited (serial $T K N$ and parallel $T \log (K N)$ additional term). The \textit{Cheap Gradient Principle} \citep{autograd} states that serial complexity of the backward pass coincides with that of the forward pass (can be extended to parallel complexity, see \citealp{parautograd1,parautograd2}).}
\label{table:1}
\begin{center}
\begin{tabular}{llll}
\textbf{METHOD} & \textbf{SERIAL TIME} & \textbf{PARALLEL TIME} & \textbf{SOLUTION DOMAIN} \\
\hline
RNN & $T N^2$ & $T \log N$ & --- \\
URNN & $T N \log N$ & $T N \log N$ & $\mathcal{U}(N)$'s subset \\
SCORNN & $T N^2 + N^3$ & $T \log N + N^2 \log N$ & $\mathcal{O}^{+1} (N) \setminus \Theta$ \\
RGD for $\mathcal{U}(N)$ & $T N^2 + N^3$ & $T \log N + N^2 \log N$ & $\mathcal{U} (N)$ \\
EXPRNN & $T N^2 + N^3$ & $T \log N + N^3$ & $\mathcal{O}^{+1} (N)$ \\
EURNN, $L$ iter. & $T L N$ & $T L$ & $\mathcal{U}(N)$ when $L = N$ \\
HR, $L$ refl. & $T L N$ & $T L \log N$ & $\mathcal{O}_L (N)$ \\
CWY, $L$ refl. (ours) & $T L N + L^2 N + L^3$ & $T \log ( L N ) + L^2 \log L$ & $\mathcal{O}_L (N)$ \\
\end{tabular}
\end{center}
\end{table*}

\begin{table*}[t]
\caption{Complexity of performing a gradient step when optimizing over $\Omega \in \text{St} (N, M)$. In the notation ``RGD-A-B'' ``A'' is C or E for canonical or Euclidean inner product choice respectively, and ``B'' is C or QR for Cayley or QR retraction respectively. The term related to computing the objective function and $\Omega$'s gradient is omitted. Parallel complexity is reported in $O(\cdot)$ notation while FLOPs are reported for the forward pass with exact constants in the leading terms. The backward pass requires only a constant time more operations (the \textit{Cheap Gradient Principle}, \citealp{autograd}). To report parallel complexity we use the same assumptions as for Table \ref{table:1}. In our estimations we use that a) a product of $d_1 \times d_2$ and $d_2 \times d_3$-sized matrix takes $2 d_1 d_2 d_3$ FLOPs \citep{hunger}, b) an inverse of $d_1 \times d_1$-sized dense and upper-triangular matrix takes $d_1^3$ and $d_1^3 / 3$ FLOPs respectively \citep{hunger}, c) QR decomposition of a $d_1 \times d_2$-sized matrix, $d_1 \geq d_2$, takes $2 d_2^2 (d_1 - \frac13 d_2)$ FLOPs \citep{qrflops} and d) eigendecomposition of a $d_1 \times d_1$-sized positive semi-definite matrix (as it is in OWN) coincides with its SVD which requires $\frac83 d_1^3$ FLOPs \citep{trefethen}. Since $N \geq M$, T-CWY needs the smallest number of FLOPs.}
\label{table:2}
\begin{center}
\begin{tabular}{llll}
\textbf{APPROACH} & \textbf{PARALLEL TIME} & \textbf{INVERTED MATRIX SIZE} & \textbf{FLOPs} \\
\hline
RGD-C-QR & $M \log (M N)$ & --- & $10 N M^2 - 2 M^3 / 3$ \\
RGD-E-QR & $M \log (M N)$ & --- & $14 N M^2 - 2 M^3 / 3$ \\
RGD-C-C & $\log (M N) + M^2 \log M$ & $2 M \times 2 M$ & $28 N M^2 + 16 M^3$ \\
RGD-E-C & $\log (M N) + M^2 \log M$ & $3 M \times 3 M$ & $72 N M^2 + 25 M^3$ \\
OWN & $\log (M N) + M^3$ & --- & $4 N M^2 + 14 M^3 / 3$ \\
T-CWY (ours) & $\log (M N) + M^2 \log M$ & $M \times M$ upper-triangular & $\boldsymbol{4 N M^2 + 7 M^3 / 3}$ \\
\end{tabular}
\end{center}
\end{table*}

We compare the serial and parallel runtime complexity of different methods to train orthogonal RNNs in Table \ref{table:1} (we introduce the notation $\mathcal{O}_L (N)$ later in this section). We also show the 
domain covered by each optimization approach.

Row ``RNN" indicates the complexity of an unconstrained RNN. \citet{hh} show that any RNN with a unitary transition matrix can be modelled by a different network with orthogonal weights. Hence, we opt for simplification by only covering the orthogonal group $\mathcal{O} (N)$. As noted by \citep{wisdom}, URNN parametrization is not enough to cover all matrices from $\mathcal{U}(N)$, which is an $N^2$-dimensional manifold.

RGD, SCORNN and EXPRNN employ a costly $O( N^3 )$ operation of matrix exponent or Cayley transform. Note that the limitation of EXPRNN covering only $\mathcal{O}^{+1} (N)$ can be alleviated, since a matrix $Q \in \mathcal{O}^s (N)$ can be parametrized by $\widehat{Q} \in \mathcal{O}^{-s} (N)$ obtained by inverting one of $Q$'s rows.

EURNN enables a tradeoff between computational complexity and unitary matrix coverage. Matrix-vector product with $F^{(i)}$ can be efficiently computed in serial time $O (N)$ (parallel $O(1)$). Next, by choosing bigger $L$, we can increase the family of supported unitary matrices at the cost of additional computation time. Eventually, when $L = N$, all unitary matrices are covered. Similar properties hold for HR decomposition -- applying a Householder reflection to a vector is an $O (N)$ (parallel $O(\log N)$) operation and the following theorem holds:
\begin{theorem}[adapted from \citealp{hh}] \label{lemma:1}
Let $Q \in \mathcal{O}^s(N)$ where $s = (-1)^N$. Then there exist nonzero $v^{(1)}, \dots, v^{(N)} \in \mathbb{R}^N$ s.t. $Q = H(v^{(1)}) \dots H(v^{(N)})$.
\end{theorem}
Although EURNN and HR methods don't have an $O(N^3)$ term in runtime complexity, they cannot be parallelized in $L$, the number of sequentially applied operators $F^{(i)}$ or $H(v^{(i)})$. This becomes a problem when $N$ is big and, thus, bigger $L$ is needed to obtain good expressiveness. We use the notation $\mathcal{O}_L (N)$ for the set of orthogonal matrices which can be obtained with $L$ Householder reflections: $\mathcal{O}_L (N) = \{ H(v^{(1)}) \dots H(v^{(L)}) \, | \, \forall i:  v^{(i)} \in \mathbb{R}^N \setminus \{ \mathbf{0} \} \}$.

Table \ref{table:2} summarizes the runtime complexity of Stiefel manifold optimization approaches. OWN requires an eigenvalue decomposition of a dense $M \times M$-sized matrix which is a cubic operation. See Appendix Section \ref{sec:woodbury} for additional discussion of RGD-based methods' runtime complexity.

\section{EFFICIENT $\mathcal{O} (N)$ and $\mathrm{St} (N, M)$ PARAMETRIZATION} \label{sec:prop}

We define the CWY transform and demonstrate its utility for RNN training. Next, we introduce a novel T-CWY map, and for both transforms prove stochastic-optimization convergence guarantees.

\subsection{Compact WY (CWY) Transform}

We suggest an alternative algorithm to compute the composition of $L$ Householder reflections. Our approach can compute a series of reflections in parallel on GPU or TPU thus increasing the effectiveness of RNN rollout in terms of floating point operations per second (FLOPS). The approach is called the \textit{compact WY} (CWY) transform \citep{joffrain}, and to our knowledge, has not been applied previously in machine learning. \citet{hh} used CWY only for theoretical reasoning about backpropagation -- they used the explicit Householder series in experiments.
\begin{theorem}[adapted from \citealp{joffrain}] \label{lemma:2}
Let $v^{(1)}, \dots v^{(L)} \in \mathbb{R}^N$ be nonzero vectors. Then 
\begin{equation} \label{eq:dec}
    H(v^{(1)}) \dots H(v^{(L)}) = I - U S^{-1} U^\top ,
\end{equation}
where $U = \begin{bmatrix} v^{(1)}/\| v^{(1)} \|_2  \dots  v^{(L)}/\| v^{(L)} \|_2 \end{bmatrix} \in \mathbb{R}^{N \times L}$ , and $S = \frac{1}{2} I + \text{striu} (U^\top U)$ where $\text{striu} (\cdot)$ returns an argument matrix with all diagonal and lower-triangular elements zeroed out.
\end{theorem}

We store $v^{(1)}, \dots, v^{(L)}$ as learnable parameters. An efficient way to do a forward pass with CWY-based RNN is as follows. We don't compute and store $Q = I - U S^{-1} U^\top$ explicitly. Instead, before each RNN rollout, we precompute $U$ and $S^{-1}$ and expand Equation (\ref{eq:rnn}, left) into the following computations: $u_t := U^\top h_{t - 1}$, $v_t := S^{-1} u_t$, $y_t := h_{t - 1} - U v_t + b$, which has two matrix-vector products with matrices of size $L \times N$ and $N \times L$. Altogether this results in the complexity estimate shown in Table \ref{table:1}. The latter approach is asymptotically efficient when $L < N$, while when $L = N$ we precompute the transition matrix (\ref{eq:dec}) into $Q$ and then perform the RNN rollout as usual.

The better parallelization pattern of CWY comes with a price of an $L^2 \log L$ term related to inverting the $S$ matrix. In practice, we find that for moderate $L$ this addition is comparable to the rollout cost, considering also that $S$ is upper-triangular and, hence, takes less FLOPs to invert \citep{hunger}. 

\subsection{Extension: Truncated CWY (T-CWY)}

We extend our approach and propose, to our knowledge, a novel parametrization of the Stiefel manifold $\text{St}(N, M)$ which we call the \textit{truncated CWY} (T-CWY) transform. We parametrize the Stiefel manifold $\text{St}(N, M)$ with $M < N$ by $\mathbb{R}^{N \times M}$ minus a zero-measure set.
\begin{theorem} \label{lemma:3}
Consider $M < N$ and a function $\gamma_{N,M}: (\mathbb{R}^N \setminus \{ \mathbf{0} \})^M \to \mathbb{R}^{N \times M}$ defined as follows. For $v^{(1)}, \dots v^{(M)} \in \mathbb{R}^N$ construct a matrix $U = \begin{bmatrix} v^{(1)}/\| v^{(1)} \|_2 & \dots & v^{(M)}/\| v^{(M)} \|_2 \end{bmatrix} \in \mathbb{R}^{N \times M}$ and assign $\gamma_{N,M} (v^{(1)}, \dots v^{(M)}) = \begin{bmatrix} I & \mathbf{0} \end{bmatrix}^\top - U S^{-1} U_1^\top \in \mathbb{R}^{N \times M}$ where $U_1$ is an upper $M \times M$ submatrix of $U$ and $S = \frac{1}{2} I + \text{striu} (U^\top U)$.
Then $\gamma_{N,M}$ is a surjective mapping to $\text{St}(N, M)$.
\end{theorem}

In other words, Theorem \ref{lemma:3} states that Stiefel matrices can be parametrized by taking $M$ first columns of a  $N \times N$ CWY-parametrized matrix with $L = M$, but without forming this $N \times N$ matrix explicitly. Computational complexity of T-CWY is indicated in Table \ref{table:2}. T-CWY is fully-parallelizable in $N$ with the number of floating point operations smaller than for any other approach due to the inverted matrix $S$ size $M \times M$ and upper-triangular structure \citep{hunger}.

\subsection{SGD Convergence Analysis}

Consider a function $f: \mathcal{O} (N) \to \mathbb{R}$ (e. g. an empirical risk) which is accessed through its stochastic proxy $\widetilde{f}$ (e. g. a minibatch loss). We prove a standard result \citep{srgd,bottou} stating that CWY-based stochastic optimization can get arbitrarily 
close to a stationary point where $\nabla f = \mathbf{0}$. For convenience we formulate our results in terms of a Householder decomposition which is equivalent to CWY.

\begin{theorem} \label{th:conv}
Let $f: \mathbb{R}^{N \times N} \to \mathbb{R}$ be a differentiable function with Lipschitz-continuous gradients on $\mathcal{O} (N)$: $\forall X', X'' \in \mathcal{O} (N): \| \nabla f (X') - \nabla f (X'') \|_F \leq M_1 \| X' - X'' \|_F$ for some $M_1 > 0$ ($\| \cdot \|_F$ denotes Frobenius norm). Let $\widetilde{f}: \mathbb{R}^{N \times N} \to \mathbb{R}$ be a stochastic differentiable function such that $\forall X \in \mathcal{O} (N): \mathbb{E} \nabla \widetilde{f} (X) = \nabla f (X)$ and suppose there exists $M_2 > 0$ such that $\forall X \in \mathcal{O} (N): \mathbb{E} \| \nabla \widetilde{f} (X) \|_F^2 \leq M_2$. Consider a sequence $\{ (v^{(k,1)} \in \mathbb{R}^N, \dots, v^{(k,L)} \in \mathbb{R}^N) \}_{k = 0}^\infty$ where $v^{(0,1)}, \dots, v^{(0,L)} \in \mathbb{R}^N$ are deterministic and nonzero and for all $k > 0, 1 \leq l \leq L$: $v^{(k,l)} = v^{(k - 1,l)} - k^{-0.5} \nabla_{v^{(k - 1,l)}} \widetilde{f} (H(v^{(k - 1,1)}) \dots H(v^{(k - 1,L)}))$.
Then all $\{ v^{(k,l)} \}$ are well-defined and for any $\epsilon > 0$,
\begin{gather*}
    \min_{0 \leq k' < K} \sum_{l = 1}^L \mathbb{E} \| \nabla_{v^{(k',l)}} f (H(v^{(k',1)}) \times \dots \\
    \times H(v^{(k',L)})) \|_2^2 = o (K^{-0.5 + \epsilon}) .
\end{gather*}
\end{theorem}
Observe that an identical result holds for T-CWY parametrization. Indeed, using notation of Theorem \ref{lemma:3} for any $f: \text{St} (N, M) \to \mathbb{R}$ $f (\gamma_{N,M} (v^{(1)}, \dots, v^{(M)})) = f ((H (v^{(1)}) \dots H (v^{(M)}))_{:,:M})$. Gradient Lipschitz-continuity of $f$ and bounded variance of $\widetilde{f}$ hold for composite functions $f ((\cdot)_{:,:M})$ and $\widetilde{f} ((\cdot)_{:,:M})$ which are plugged into Theorem \ref{th:conv} to get analogous result for T-CWY. The proof of Theorem \ref{th:conv}, as well as a high-level sketch to help intuition, can be found in Appendix \ref{sec:th4pr}.

\subsection{Convolutional Non-Exploding Recurrent Unit (ConvNERU)}

Based on the proposed Stiefel matrix parametrization, we introduce 
a \textit{convolutional non-exploding recurrent unit} (ConvNERU) -- a recurrent module which is provably resistant to gradient and hidden state explosion. 
Given a sequence of images $X_1, \dots, X_T \in \mathbb{R}^{h \times w \times f_{in}}$, our proposed module is the following modification of (\ref{eq:rnn}): $Y_t := \mathcal{K} * G^{(t - 1)} + B, \quad G^{(t)} := \sigma( Y_t + \mathcal{K}^{in} * X_t )$ where $G^{(0)}, \dots, G^{(T)} \in \mathbb{R}^{h \times w \times f_{out}}$ are hidden states, $B \in \mathbb{R}^{h \times w \times f_{out}}$ is a bias tensor which is parametrized by $b \in \mathbb{R}^{f_{out}}$ so that $b = B_{i,j}$ for any $i, j$, $\sigma$ is an element-wise nonlinearity, ``$*$" denotes convolution operation and $\mathcal{K} \in \mathbb{R}^{q \times q \times f_{out} \times f_{out}}, \mathcal{K}^{in} \in \mathbb{R}^{q \times q \times f_{in} \times f_{out}}$ are convolution kernels with $q$ being kernel size. Denote by $\widehat{\mathcal{K}}$ a $(q^2 f_{out} \times f_{out})$-sized matrix such that for any $l, p \leq q$ and $i, j \leq f_{out}$ it holds that $\widehat{\mathcal{K}}_{l q f_{out} + p f_{out} + i,j} = \mathcal{K}_{l,p,i,j}$. We equip ConvNERU with a constraint $(q \widehat{\mathcal{K}}) \in \text{St} (q^2 f_{out}, f_{out})$ which is implemented by T-CWY parametrization. In Appendix Section \ref{sec:convneru}, we theoretically show that ConvNERU is resistant to  norm explosion.

\section{EXPERIMENTS} \label{sec:exp}

We evaluate CWY on standard benchmarks and a neural machine translation setup. Then, we evaluate T-CWY and ConvNERU on a video prediction setup.

\begin{figure*}[!ht]

\begin{subfigure}[b]{0.38\textwidth}
  \centering
  \includegraphics[width=\linewidth]{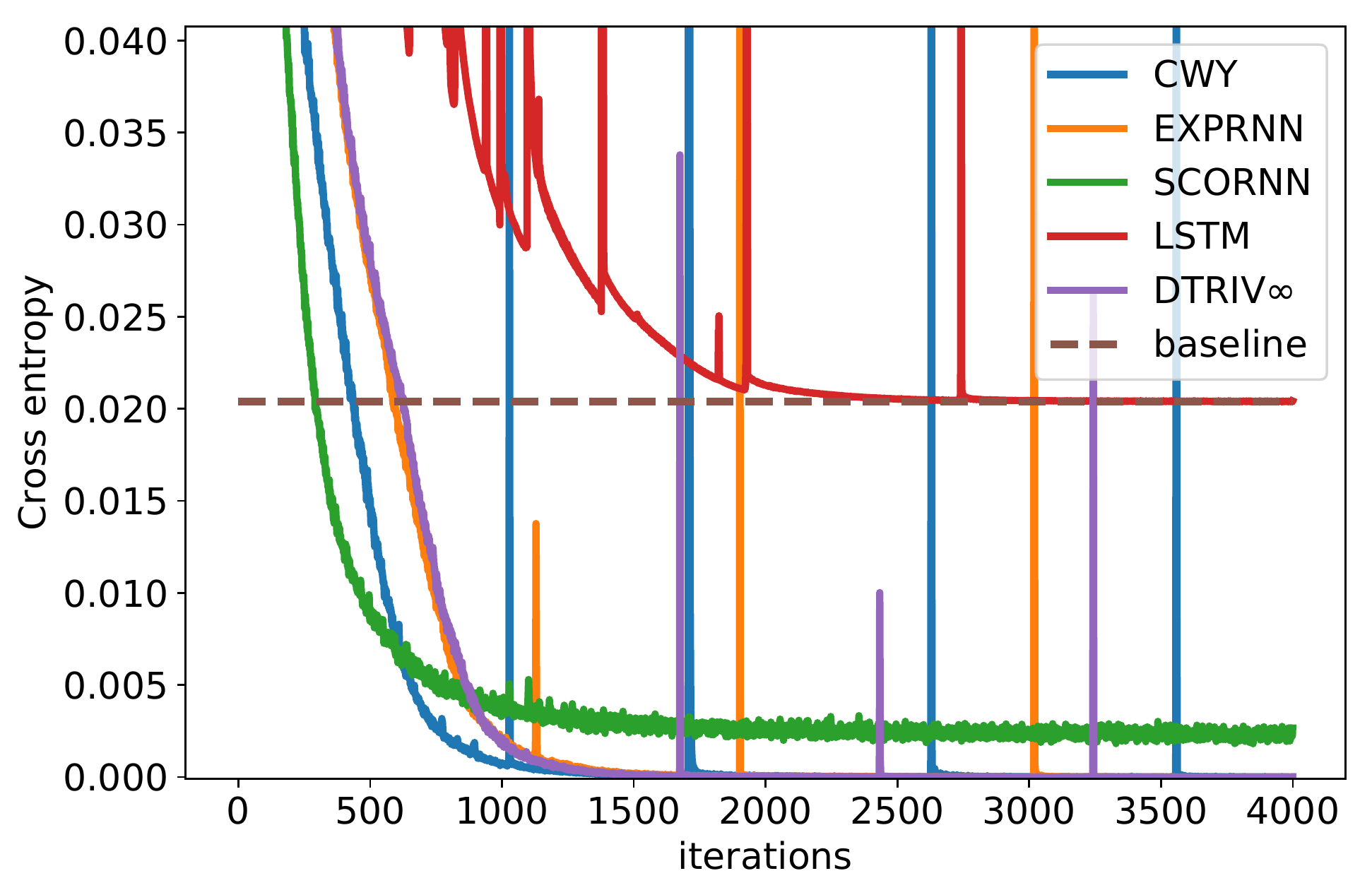}
  \caption[]%
        {{\small \,}}   
  \label{fig:copy}
\end{subfigure}%
\begin{subfigure}[b]{0.38\textwidth}
  \centering
  \includegraphics[width=\linewidth]{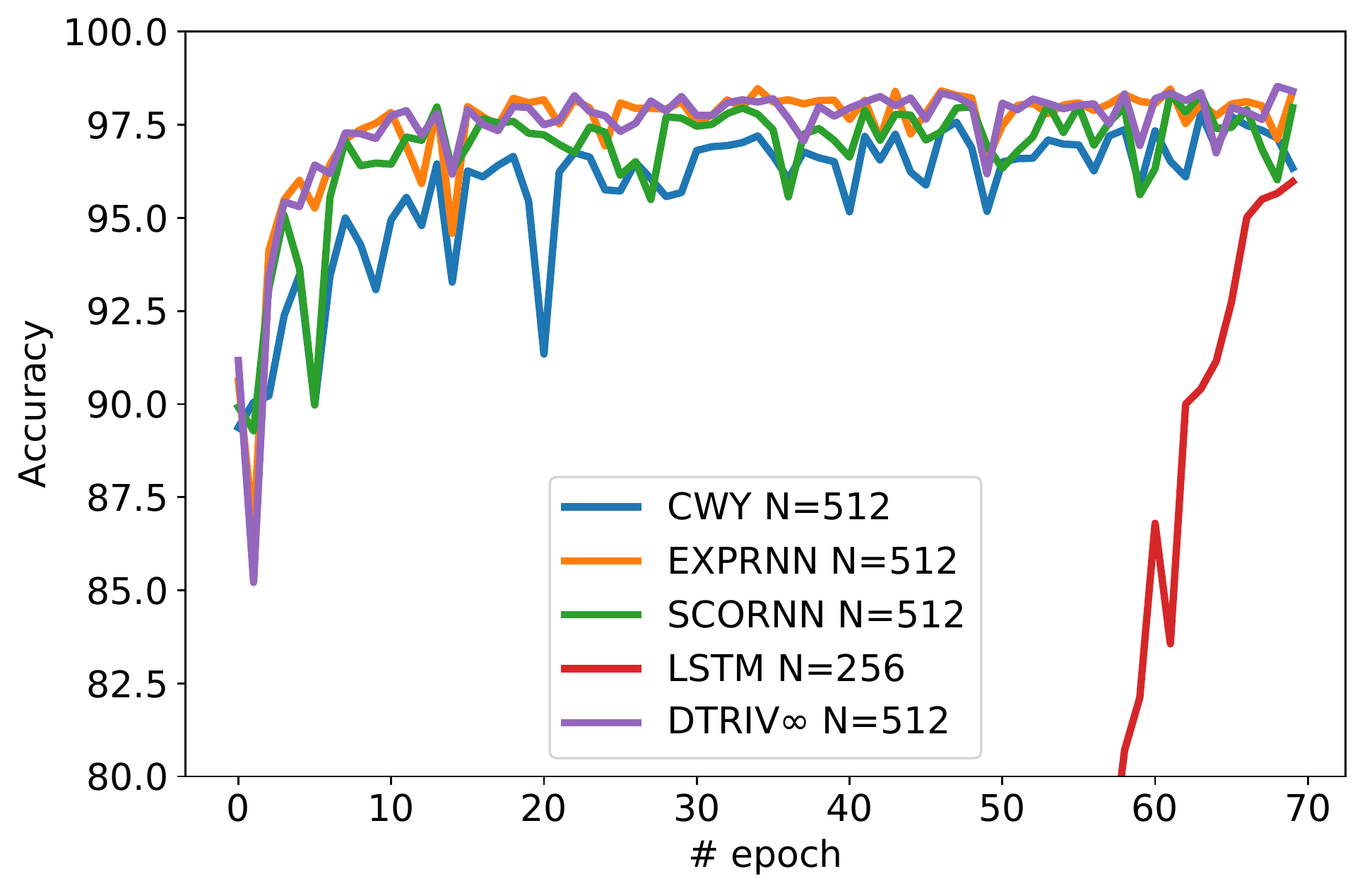}
  \caption[]%
        {{\small \,}}   
  \label{fig:mnist}
\end{subfigure}
\begin{subfigure}[b]{0.18\textwidth}
  \centering
  \includegraphics[width=\linewidth]{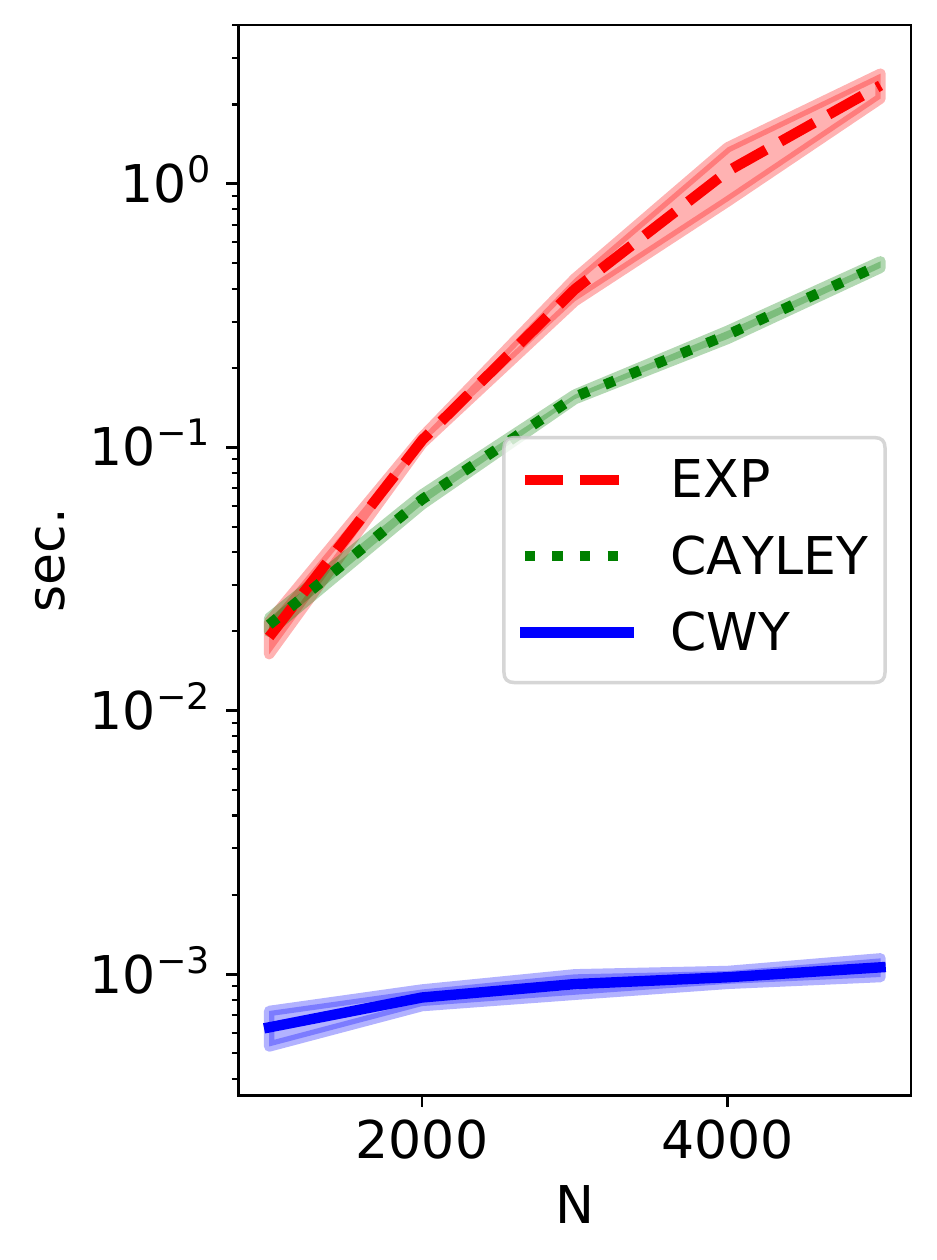}
  \caption[]%
        {{\small \,}}   
  \label{fig:timecomp}
\end{subfigure}
\caption{\textbf{(a)} Copying task, $\mathcal{T} = 1000$. \textbf{(b)} Pixel-by-pixel MNIST, test accuracy. \textbf{(c)} parametrization time comparison, mean and standard error over $10$ samples.}
\label{fig:sttasks}
\end{figure*}

\subsection{Standard Tasks and Time Comparison}

We evaluate orthogonal RNN with CWY parametrization on standard benchmarks, aimed to test the ability of  RNN to capture long-term dependencies in the data:

1. \textit{Copying task}. The input contains $10$ digits sampled uniformly from $\{ 1, \dots, 8 \}$, then $\mathcal{T}$ zeros, one ``$9$'' (\textit{start}) and $9$ zeros. The output consists of $\mathcal{T} + 10$ zeros and $10$ first digits from the input. Hence, the goal of RNN is to copy the random input prefix after observing $\mathcal{T}$ zeros. The goal is to beat a no-memory \textit{baseline}, which outputs $\mathcal{T} + 10$ zeros and $10$ randomly sampled digits from $\{ 1, \dots, 8 \}$ independently of the input. The cross-entropy of this baseline is $10 \log 8 / (\mathcal{T} + 20)$.

2. \textit{Pixel-by-pixel MNIST}. The input contains images of digits from MNIST \citep{mnist}, flattened into sequences of length $784$. The goal is to classify the digit using the last hidden state of the RNN.

For both experiments we reuse the publicly available code from \citep{cheap} in PyTorch \citep{pytorch}, \textbf{without tuning any hyperparameters}, changing random initializations or seeds, etc. Figures \ref{fig:copy}, \ref{fig:mnist} (a-b) demonstrates the results of plugging CWY directly into the code. In the Copying task with $\mathcal{T} = 1000, L = N = 190$, CWY is converging to zero cross entropy faster, than EXPRNN and DTRIV$\infty$ \citep{dtriv}, while SCORNN fails to converge to zero and LSTM \citep{lstm} cannot beat the baseline. In the Pixel-by-pixel MNIST, CWY ($L = N$) shows competitive performance, going beyond $95\%$ accuracy and matching the results of \citet{hh}. See details and additional experimental results (Copying task with $\mathcal{T} = 2000$ and permuted MNIST) in Appendix \ref{sec:sttasksmore}.

In addition to standard benchmarks, we perform a time comparison for computing CWY, exponential parametrization and Cayley map (Figure \ref{fig:timecomp}), where the argument is a random matrix. See Appendix \ref{sec:sttasksmore} for details. We conduct experiments on GPU and use the following methods from PyTorch 1.7: \texttt{torch.matrix\_exp} implementing a state-of-the-art algorithm for matrix exponential \citep{mexp}, \texttt{torch.solve} for Cayley map and \texttt{torch.triangular\_solve} for CWY. We observe that for a range of matrix sizes CWY is 1-3 orders of magnitude faster than other parametrizations. While we used full CWY ($L = N$) for this comparison, $L < N$ would lead to further speedups.

\subsection{Neural Machine Translation} 


We train an orthogonal RNN-based seq2seq model with attention mechanism \citep{bahdanau2014neural} to translate sentence pairs between a given source and target language. See Appendix Section \ref{sec:nlpdetails} for additional architectural and experimental details. We focus on the English-to-Spanish dataset within the Tatoeba corpus \citep{artetxe2019massively}, a publicly available dataset with over 100,000 sentence pairs. We compare several variants of orthogonal RNNs with absolute value nonlinearities which are exact norm-preserving \citep{dorobantu2016dizzyrnn} and compare them against GRUs and LSTMs used as RNN units in a seq2seq architecture. All variants of RNN have hidden dimension $N=1024$. For the CWY and non-orthogonal variants, we conduct experiments with the Adam optimizer (see Table \ref{tab:spa-to-eng}).

We find that standard RNNs underperform LSTMs and GRUs \citep{gru}, but that parametrization-based orthogonal RNN variants are able to achieve comparable performance.
Among orthogonal RNN methods, our CWY approaches achieve the lowest test cross-entropy, whilst requiring the fewest parameters and, via our efficient parametrization, retaining training speed comparable to LSTMs and GRUs. 
We find that even the full-orthogonal CWY scheme with $L = N$ runs faster in practice than other orthogonal approaches.
A sweet-spot parameter value $L=128$ illustrates the trade-off between the capacity of the model (which increases with larger values of $L$) and the landscape of the objective function (that simplifies with smaller values of $L$). As mentioned before, in exact arithmetic our CWY is equivalent to the explicit Householder reflections approach leveraged by \citet{joffrain}; however, our approach achieves far superior speed, as illustrated in Table \ref{fig:hr-cwy-speed}. The enhanced speed of our CWY variants, when paired with the optimizer-choice flexibility, makes this approach a compelling alternative to LSTMs and GRUs.  

\begin{figure}[t]
\centering
\centerline{\includegraphics[width=\linewidth]{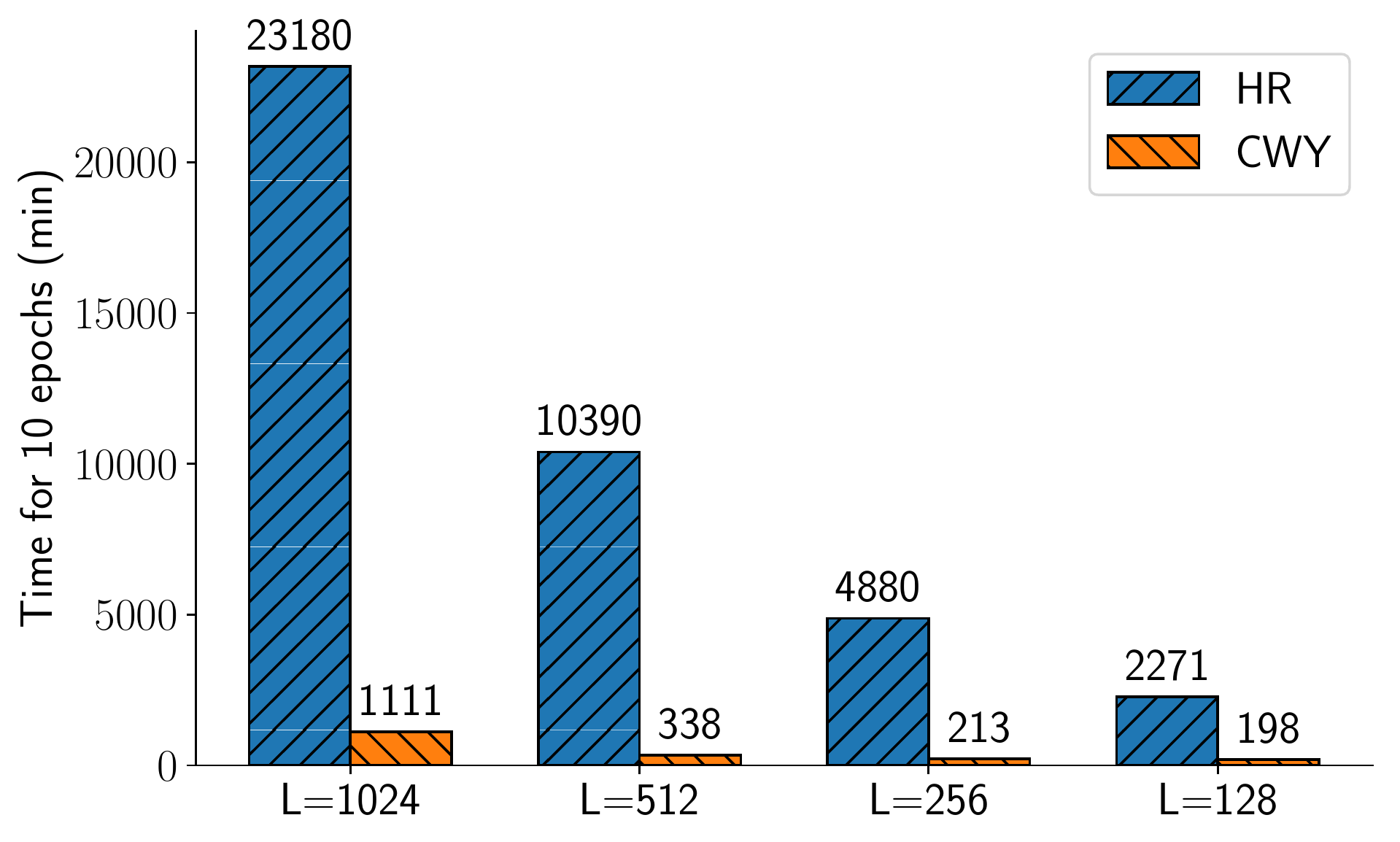}}
\caption{The CWY and HR methods are numerically equivalent; however, the parametrization of the CWY allows us to perform projections much more efficiently, leading to dramatic improvements in training time and, thereby, practical viability. The experiment is conducted on a Tensor Processing Unit (TPU).}
\label{fig:hr-cwy-speed}
\end{figure}

\begin{table}[t]
\caption{Tatoeba Spa-to-Eng NMT results. We report perplexity (PP) on a test set (a smaller value indicates a better result). Time is reported for 10 epochs. CWY achieves the best performance while preserving speed and requiring the fewest parameters. There is a sweet-spot for the test loss ($L=128$).}
\label{table:NMT}
\begin{center}
\begin{tabular}{llll}
\textbf{MODEL} & \makecell{\textbf{TEST} \\ \textbf{PP}} & \makecell{\textbf{TIME} \\ \textbf{(MIN.)}} & \textbf{PARAMS} \\
\hline
RNN & 1.66 & 148 & $\approx$ 25M \\
GRU & 1.47 & 173 & $\approx$ 32M  \\
LSTM & 1.46 & 232 & $\approx$ 37M \\
SCORNN & 1.49 & 1780 & $\approx$ 25M \\
RGD & 4.03 & 1780 & $\approx$ 25M \\
EXPRNN & 1.51 & 2960 & $\approx$ 25M \\
CWY L=1024 & 1.47 & 1111 & $\approx$ 25M \\
CWY L=512 & 1.58 & 338 & $\approx$ 24M \\
CWY L=256 & 1.56 & 213 & $\approx$ 23M \\
CWY L=128 & \textbf{1.41} & \textbf{198} & $\approx$ \textbf{23M} \\
CWY, L=64 & 1.52 & 175 & $\approx$ 23M \\
\label{tab:spa-to-eng}
\end{tabular}
\end{center}
\end{table}

\subsection{Video prediction with ConvNERU}

\begin{table*}[t]
\caption{KTH action dataset test results. The indicated metric is average per-frame $l_1$-loss. Video frames are in grey scale with brightness ranged in $[0, 1]$. The GPU memory is evaluated for the ``Box[ing]" class which has the longest sequences. We do not report the last two columns for the ``Zeros" method which is only aimed to demonstrate the importance of recurrent connections.}
\label{table:exp1}
\begin{center}
\begin{tabular}{lllllllll}
\textbf{METHOD} & \textbf{WALK}
    & \textbf{JOG} & \textbf{RUN} &
    \textbf{BOX} & \textbf{WAVE} &  \!\!\!\! \textbf{CLAP} & \!\!\!\! \textbf{\# PARAMS} & \!\!\!\! \textbf{GPU MEMORY} \\
\hline
ConvLSTM & 223.3 & 266.8 & 297.8 & 188.9 & 157.9 & 162.3 & $\approx$ 3.26 M & 8.7 Gb \\
Zeros & 160.3 & 176.1 & 203.8 & 179.0 & 197.2 & 147.4 & --- & --- \\
Glorot-Init & 145.8 & 161.5 & 182.1 & 179.9 & 164.5 & 145.4 & $\approx$ 0.72 M & 3.5 Gb \\
Orth-Init & 139.9 & 153.2 & 175.0 & 173.3 & 150.8 & 144.0 & As above & As above \\
RGD-C-C & 135.8 & 155.7 & 170.7 & 172.9 & 160.3 & 144.5 & As above & As above \\
RGD-E-C & 143.3 & 152.5 & 173.7 & 171.9 & 172.9 & 142.6 & As above & As above \\
RGD-C-QR & 143.1 & 155.0 & 171.5 & 173.1 & 150.2 & 142.7 & As above & As above \\
RGD-E-QR & 135.5 & 153.9 & 169.6 & 169.9 & 160.4 & 142.5 & As above & As above \\
RGD-Adam & 142.6 & 157.3 & 177.8 & 176.8 & 159.1 & 145.2 & As above & As above \\
OWN & 137.5 & 155.0 & 177.7 & 171.3 & 149.8 & 142.5 & As above & As above \\
T-CWY & \textbf{134.6} & \textbf{149.8} & \textbf{166.7} & \textbf{166.2} & \textbf{147.8} & \textbf{141.2} & As above & As above \\
\end{tabular}
\end{center}
\end{table*}

\begin{figure*}
\centering
\includegraphics[width=\linewidth]{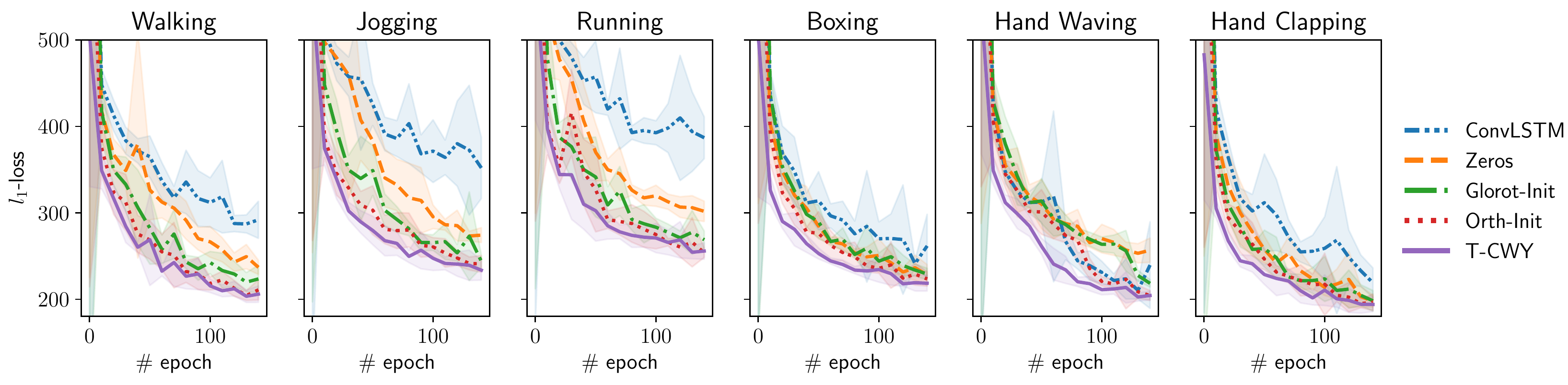}
\caption{Validation $l_1$-loss. Mean and standard error across each $10$ epochs is reported.}
\label{fig:learning_curve}
\end{figure*}

We demonstrate performance of T-CWY and ConvNERU in the task of one-step-ahead video prediction on the KTH action dataset. As a baseline we chose ConvLSTM \citep{convlstm}, a convolutional adaptation of LSTM. 
In addition, our goal is to compare with other methods for Stiefel optimization and justify the need for Stiefel constraints.

We conduct experiments on the KTH action dataset \citep{kth} containing grey scale video recordings of 25 people, each performing 6 types of actions: walking, jogging, running, boxing, hand waving and hand clapping. We do separate evaluations for each action type to evaluate how the model learns different types of dynamics. As a video-prediction architecture we apply a simplified version of \citep{arch1,arch2} where we try different types of recurrent block design (see further). We opt for minimizing the $l_1$-loss $| \widehat{\mathcal{I}} - \mathcal{I} |$ ($l_1$-loss) during training where $\widehat{\mathcal{I}}, \mathcal{I}$ denote predicted and ground-truth frame respectively. For all unconstrained parameters we use the Adam optimizer. See Appendix Section \ref{sec:vpdetails} for more details on data preprocessing, experiment setup and architecture. We compare different designs of recurrent unit used in the full architecture. \textit{ConvLSTM} was used in the original variant of the architecture \citep{arch1,arch2}. \textit{Zeros} indicates ConvNERU with transition kernel $K$ zeroed out (i.e. prediction conditioned on the previous frame only). \textit{Glorot-Init} is a modified ConvNERU where $K$ is unconstrained initialized through Glorot uniform initialization \citep{glorot}. \textit{Orth-Init} indicates a modified ConvNERU with unconstrained $q \widehat{K}$ initialized as a Stiefel matrix by QR decomposition of a random matrix. \textit{RGD-*-*} indicates Stiefel RGD for optimizing $q \widehat{K}$ with various combinations of inner product and retractor (consistent with the notation in Table \ref{table:2}). \textit{RGD-Adam} is an Adam adaptation of RGD \citep{rgdst} applied to optimization of $q \widehat{K}$. Finally, \textit{OWN} and \textit{T-CWY} indicate ConvNERU with $q \widehat{K}$ matrix parametrized by OWN and T-CWY respectively.

Table \ref{table:exp1} demonstrates test $l_1$-loss, number of parameters and maximal GPU memory consumption. Additionally, Figure \ref{fig:learning_curve} demonstrates validation $l_1$-loss depending on epoch number for a subgroup of evaluluated methods. We see from the figure that in most cases, with the same learning rate, ConvLSTM cannot outperform ``Zeros" baseline which has no recurrence and, hence, does not face an issue of gradient explosion or vanishing. Among the versions of ConvNERU and its unconstrained analogs, we observe that T-CWY performs best on both validation and test set while having several times less parameters and using much less GPU memory than ConvLSTM.

\section{CONCLUSION}

We introduced an efficient scheme for parametrizing orthogonal groups $\mathcal{O} (N)$ and Stiefel manifolds $\text{St} (N, M)$, and compared to earlier approaches. The proposed $\mathcal{O}(N)$-parametrization scheme is efficient when working with large-scale orthogonal matrices on a parallelized computation unit such as GPU or TPU. We empirically demonstrated strong performance in real-world applications.

\section*{Acknowledgements}

We thank Jonathan Gordon, Wessel Bruinsma and David Burt for helpful feedback on an early version of the manuscript. We further thank anonymous reviewers for their valuable feedback.

Valerii Likhosherstov acknowledges support from the Cambridge Trust and DeepMind. Adrian Weller acknowledges support from the David MacKay Newton research fellowship at Darwin College, The Alan Turing Institute under EPSRC grant EP/N510129/1 and U/B/000074, and the Leverhulme Trust via CFI.

\bibliography{references}

\newpage

\onecolumn
\LARGE
\textsc{Supplementary Materials for the Paper ``CWY Parametrization: a Solution for Parallelized Optimization of Orthogonal and Stiefel Matrices''}

\normalsize
\appendix

\- \\
In this Appendix we provide the following:
\begin{itemize}
    \item In Section \ref{sec:woodbury}, Stiefel RGD through the Sherman-Morrison-Woodbury formula
    \item In Section \ref{sec:convneru}, Hidden state gradients of ConvNERU
    \item In Section \ref{sec:sttasksmore}, Copying, pixel-by-pixel MNIST and time comparison details
    \item In Section \ref{sec:nlpdetails}, Neural machine translation details
    \item In Section \ref{sec:vpdetails}, Video prediction details
    \item In Section \ref{sec:proofs}, Proofs of results
\end{itemize}

\section{STIEFEL RGD THROUGH THE SHERMAN-MORRISON-WOODBURY FORMULA} \label{sec:woodbury}

RGD with Cayley retraction requires inverting the $N \times N$-sized matrix $\eta_k A^{(k - 1)}$ thus becoming cubic in $N$. To make the computation more tractable, \cite{tagare} proposes to use the Sherman-Morrison-Woodbury formula which reduces the size of the inverted matrix to $2 M \times 2 M$ when the canonical inner product is chosen for RGD. A straightforward extension of Tagare's approach to the Euclidean inner product would require to invert $3M \times 3M$-sized matrix. To demonstrate that, we adapt derivations of \cite{tagare} for canonical inner product and extend them to Euclidean inner product. The following Lemma shows how to compute update $g_k (\eta_k)$ in time $O(N M^2 + M^3)$ without constructing $A^{(k - 1)}$ explicitly.
\begin{lemma}
Consider $\Omega \in \mathbb{R}^{N \times M}$ and $A = B C^\top \in \text{Skew} (N)$ for some matrices $B, C \in \mathbb{R}^{N \times D}$, $D \leq N$. Then
\begin{equation} \label{eq:l2}
    \text{Cayley} (A) \Omega = \Omega - B \biggl( I + \frac12 C^\top B \biggr)^{-1} \biggl( C^\top \Omega \biggr)
\end{equation}
\end{lemma}
\begin{proof}
We first need to show that the right hand side of (\ref{eq:l2}) always exists, i.e. $I + \frac12 C^\top B$ is nonsingular:
\begin{align*}
    &\det (I + \frac12 C^\top B) = \det (I + \frac12 B C^\top) = \det (I + \frac12 A) \neq 0
\end{align*}
where in the first transition we apply Sylvester's determinant identity. $I + \frac12 A$ is nonsingular, because the spectrum of any skew-symmetric matrix is pure-imaginary (Theorem 12.9 from \citealp{gallier}). So the right hand side is well defined.

Through the application of Sherman-Morrison-Woodbury formula we deduce that
\begin{align*}
    \text{Cayley} (A) \Omega &= \biggl( I + \frac12 B C^\top \biggr)^{-1}  \biggl( I - \frac12 B C^\top \biggr) \Omega \\
    &= \biggl( I - \frac12 B (I + \frac12 C^\top B)^{-1} C^\top \biggr) \biggl( I - \frac12 B C^\top \biggr) \Omega \\
    &= \Omega - \frac12 B \biggl( (I + \frac12 C^\top B)^{-1} (I - \frac12 C^\top B) + I \biggr) C^\top \Omega \\
    &= \Omega - \frac12 B ( I + \frac12 C^\top B )^{-1} ( 2 I - C^\top B + C^\top B ) C^\top \Omega \\
    &= \Omega - B \biggl( I + \frac12 C^\top B \biggr)^{-1} \biggl( C^\top \Omega \biggr)
\end{align*}
which concludes the proof.
\end{proof}
For convenience denote $\mathcal{G}^{(k - 1)} = \frac{\partial f}{\partial \Omega} (\Omega^{(k - 1)})$. Depending on the inner product choice we get the following cases:

1. \textbf{Canonical inner product}. Then
\begin{align*}
    \eta_k A^{(k - 1)} &= \eta_k \mathcal{G}^{(k - 1)} {\Omega^{(k - 1)}}^\top - \eta_k \Omega^{(k - 1)} {\mathcal{G}^{(k - 1)}}^\top = B C^\top
\end{align*}
where
\begin{gather*}
    B = \eta_k \begin{bmatrix} \mathcal{G}^{(k - 1)} & \Omega^{(k - 1)} \end{bmatrix}, \quad C = \begin{bmatrix} \Omega^{(k - 1)} & -\mathcal{G}^{(k - 1)} \end{bmatrix}, \quad B, C \in \mathbb{R}^{N \times 2 M}.
\end{gather*}

2. \textbf{Euclidean inner product}. Then
\begin{align*}
    \eta_k A^{(k - 1)} &= \eta_k \mathcal{G}^{(k - 1)} {\Omega^{(k - 1)}}^\top - \eta_k \Omega^{(k - 1)} {\mathcal{G}^{(k - 1)}}^\top + \frac{\eta_k}{2} \Omega^{(k - 1)} E {\Omega^{(k - 1)}}^\top = B C^\top
\end{align*}
where
\begin{gather*}
    E = {\mathcal{G}^{(k - 1)}}^\top \Omega^{(k - 1)} - {\Omega^{(k - 1)}}^\top \mathcal{G}^{(k - 1)}, \quad B = \eta_k \begin{bmatrix} \mathcal{G}^{(k - 1)} & \Omega^{(k - 1)} & \frac12 \Omega^{(k - 1)} E \end{bmatrix}, \\
    C = \begin{bmatrix} \Omega^{(k - 1)} & -\mathcal{G}^{(k - 1)} & \Omega^{(k - 1)} \end{bmatrix}, \quad
    B, C \in \mathbb{R}^{N \times 3M}.
\end{gather*}

\section{HIDDEN STATE GRADIENTS OF CONVNERU} \label{sec:convneru}

The convolution operation can be expressed as
\begin{gather*}
(\mathcal{K} * G^{(t - 1)})_{i,j} = \widehat{\mathcal{K}}^\top \overline{G}^{(t - 1)}_{i,j}, \quad \overline{G}^{(t - 1)} \in \mathbb{R}^{h \times w \times q^2 f_{out}}, \label{eq:kresh} \\
\overline{G}^{(t - 1)}_{i,j} = \text{concat} \biggl( \{ G^{(t - 1)}_{l,p} \, | \, i - \frac{q - 1}{2} \leq l \leq i + \frac{q - 1}{2}, j - \frac{q - 1}{2} \leq p \leq j + \frac{q - 1}{2}  \} \biggr) \label{eq:concat}
\end{gather*}
where $G^{(t - 1)}_{l,p} \in \mathbb{R}^{f_{out}}$ is a zero vector when $l, p$ are pointing outside image borders (\textit{zero padding}). By definition of $\overline{G}^{(t - 1)}$ and $\mathcal{K} * G^{(t - 1)}$ we have the following chain of inequalities between Frobenius norms $\| \cdot \|_F$:
\begin{gather*}
    \| \mathcal{K} * G^{(t - 1)}) \|_F^2 = \sum_{i,j} \| (\mathcal{K} * G^{(t - 1)}))_{i,j} \|^2_2 \label{eq:ineq1} = \sum_{i,j} \| \widehat{\mathcal{K}}^\top \overline{G}^{(t - 1)}_{i,j} \|_2^2 \leq \sum_{i,j} \| \widehat{\mathcal{K}} \|_2^2 \| \overline{G}^{(t - 1)}_{i,j} \|_2^2 \\
    = \| \widehat{\mathcal{K}} \|_2^2 \cdot \| \overline{G}^{(t - 1)} \|_F^2 \leq q^2 \| \widehat{\mathcal{K}} \|_2^2 \cdot \| G^{(t - 1)} \|_F^2 \label{eq:ineq2}
\end{gather*}
Assuming that $| \sigma (x) | \leq | x |$ which holds for most popular choices of nonlinearity (ReLU, LeakyReLU, tanh), the norm of $G^{(t)}$ cannot grow in exponential manner. The same holds for a sequence of gradients with respect to $\{ G^{(t)} \}$, since it is obtained by sequentially applying a transposed linear operator corresponding to "$\mathcal{K} *$" convolution operation and transposition preserves the linear operator norm. This justifies the property of ConvNERU being robust to gradient explosion while allowing long-term information propagation thank to Stiefel convolution kernel. The conducted analysis is reminiscent of Lipschitz constant estimate for image classification CNNs performed by \cite{parseval}.

\section{COPYING, PIXEL-BY-PIXEL MNIST AND TIME COMPARISON: MORE DETAILS AND RESULTS} \label{sec:sttasksmore}

Results for Copying task ($\mathcal{T} = 2000$), and permuted MNIST (i.e. when pixels in a flatten image are permuted randomly) are shown on Figure \ref{fig:sttasks2}. For all setups but SCORNN in the Copying task we used initialization technique from \citep{henaff}, whilst for SCORNN we used initialization from \citep{helfrich}. For all setups in the Pixel-by-pixel MNIST (whether permuted or not) we used initialization from \citep{helfrich}. While our results on Pixel-by-pixel MNIST match those of \citet{hh}, \citet{hh} were not able to provide comparable results for the Copying task. We observe that correct initialization is crucial for this task.

To initialize CWY, we, first of all, initialize a skew-symmetric matrix, as discussed above. Then we take exponent of this matrix, obtaining an orthogonal matrix. Then, in order to initialize vectors $v^{(1)}, \dots, v^{(N)}$, we run the same procedure as in the Theorem \ref{lemma:1} proof (QR decomposition using Householder reflections).

To do the time comparison, we draw elements of $v^{(1)}, \dots, v^{(N)}$ for CWY from a standard normal distribution. For matrix exponent and Cayley map, we initialize skew symmetric arguments as $\mathcal{X} - \mathcal{X}^\top$, where entries of $\mathcal{X}$ are sampled from a standard normal distribution.

\begin{figure*}[!ht]
\centering
\begin{subfigure}[b]{0.49\textwidth}
  \centering
  \includegraphics[width=\linewidth]{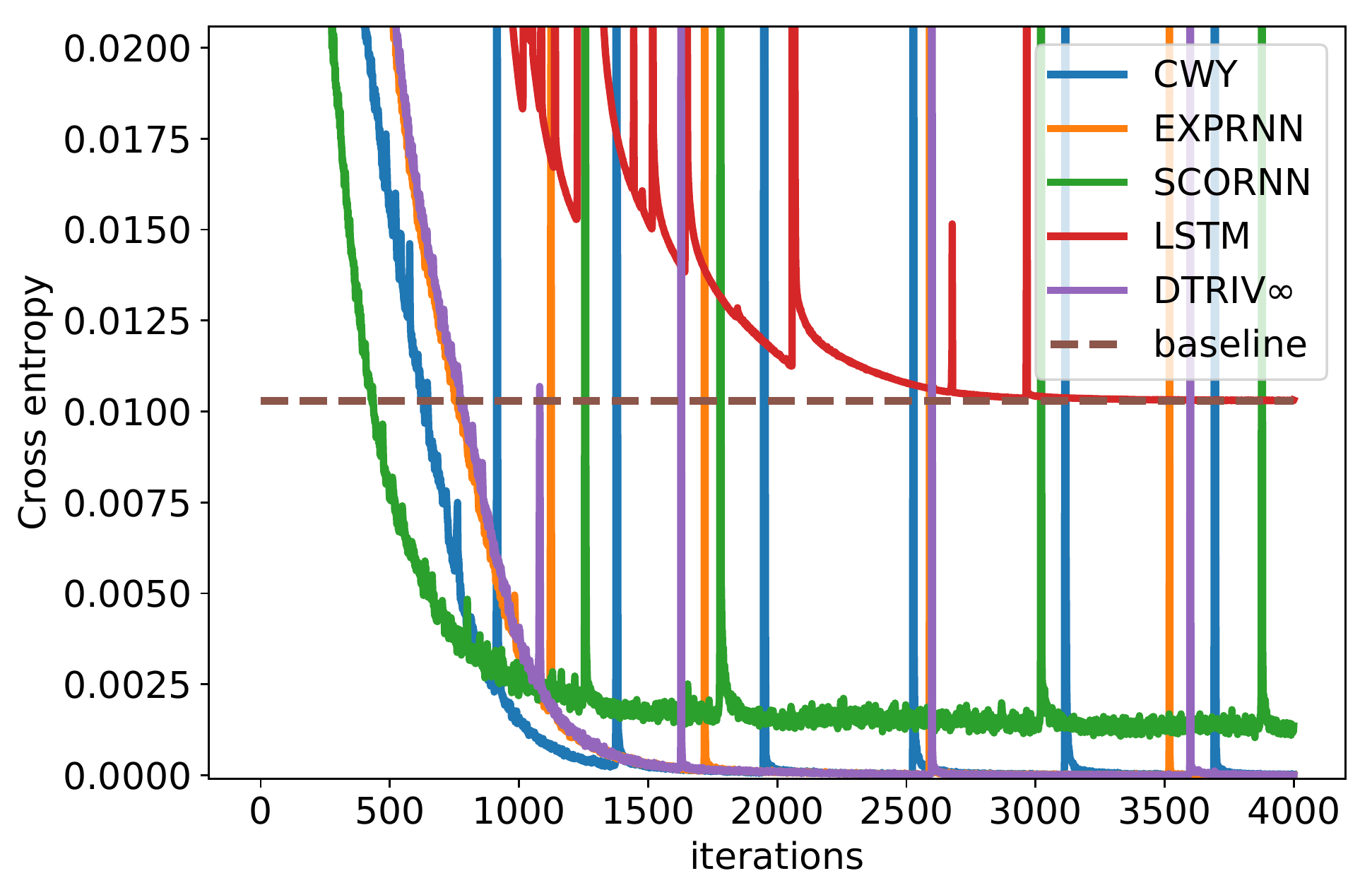}
  \caption[]%
        {{\small \,}}   
  \label{fig:copy2000}
\end{subfigure}%
\begin{subfigure}[b]{0.49\textwidth}
  \centering
  \includegraphics[width=\linewidth]{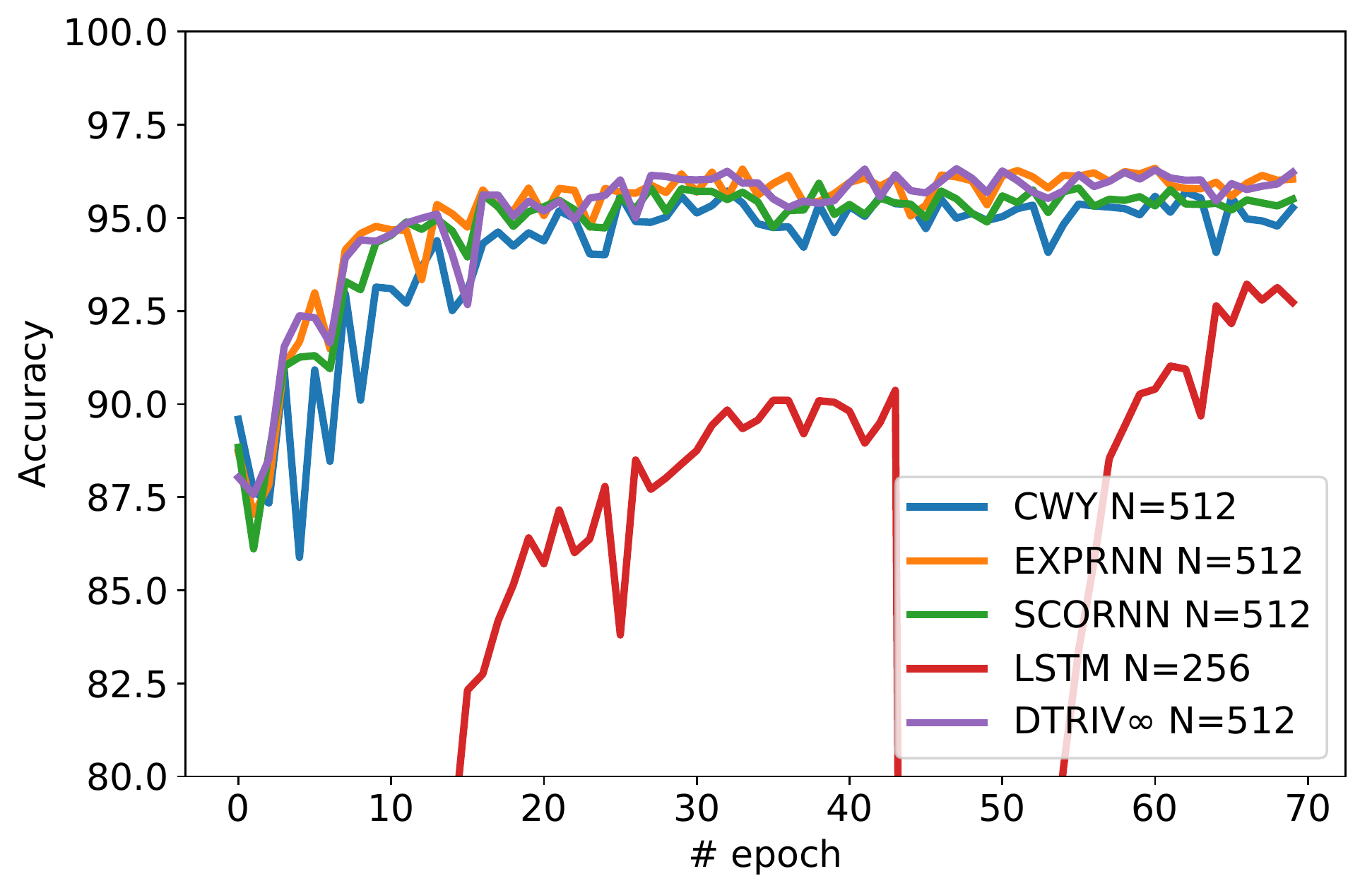}
  \caption[]%
        {{\small \,}}   
  \label{fig:pmnist}
\end{subfigure}
\caption{\textbf{(a)} Copying task, $\mathcal{T} = 2000$. \textbf{(b)} Permuted Pixel-by-pixel MNIST, test accuracy.}
\label{fig:sttasks2}
\end{figure*}

\section{NEURAL MACHINE TRANSLATION: MORE DETAILS AND RESULTS} \label{sec:nlpdetails}

We take aligned bi-texts between the source and target languages and, as preprocessing, remove accents and return word pairs in the form [English, Spanish]. The resulting dataset has an average sequence length of $\approx 17$ for both the input and target sequences.

Using a single Tensor Processing Unit (TPU) per model, we train several models from scratch, with no pre-training, on 80,000+ sentence pairs and test on the remaining 20,000+ pairs from the full 100,000+ pair dataset to compare their learning capabilities and stability. We use JAX\footnote{\href{https://jax.readthedocs.io/en/latest/}{https://jax.readthedocs.io/en/latest/}} library for the implementation. Given that we evaluate all models on the same corpus and that our goal is to benchmark across architectures, we elected to employ no pre-training and examine/compare cross-entropy loss directly.  

See Figure \ref{fig:att} for the architecture illustration. In our experiments, we used a batch size of 64, an embedding dimension size of 256, and a learning rate of $10^{-2}$. For hyperparameter sweeps, we ran experiments with smaller hidden unit sizes. We also experimented with larger and smaller learning rates. Ultimately, for simplicity and clarity, we only present results using the parameters described above.

For additional experimental results, see Table \ref{table:NLPSTATS}.

\section{VIDEO PREDICTION: MORE DETAILS} \label{sec:vpdetails}

All videos, 4 seconds in average, are recorded with a static camera with 25 fps frame rate and frame size of $160 \times 120$ pixels. We crop and resize each frame into $128 \times 128$ pixels and then reshape each frame into $64 \times 64 \times 4$ by moving groups of $2 \times 2$ pixels into channel dimension. Since each video sequence has a different number of frames, we employ zero padding during batch construction. We use persons with indices 1-12 for training, 13-16 for validation and 17-25 for testing. See Table \ref{table:stats} for KTH dataset statistics.

Given a sequence of known frames $\mathcal{I}^{(1)}, \dots, \mathcal{I}^{(t)} \in [ 0, 1 ]^{64 \times 64 \times 4}$, the network outputs a prediction $\widehat{\mathcal{I}}^{(t + 1)}$ of the next frame $\mathcal{I}^{(t + 1)}$. The network is designed as a recurrent block composed of several convolutional recurrent units stacked together with the sequence $\{ \mathcal{I}^{(i)} \}_{i = 1}^t$ passed to the input.
In order to increase the receptive field of the recurrent architecture while maintaining a tractable training procedure, we adapt a simplified version of the video prediction architecture from \cite{arch1,arch2}. Namely, we stack several recurrent units with a bottleneck structure (hidden sizes $32 \times 32 \times 32 \to 16 \times 16 \times 64 \to 8 \times 8 \times 128 \to 16 \times 16 \times 64 \to 32 \times 32 \times 32$) and skip connections. We alternate recurrent layers with strided convolutions and then deconvolutions. After each convolution and deconvolution we place a ReLU nonlinearity, as well as using ReLU as the recurrent nonlinearity $\sigma$. In the proposed architecture a prediction $\widehat{\mathcal{I}}^{(t + 1)}$ is conditioned upon $\mathcal{I}^{(t)}$ through bottleneck and skip connections and conditioned upon $\{ \mathcal{I}^{(t')} \}_{t' < t}$ through recurrent temporal connections. See Figure \ref{fig:arch} for architecture illustration.

We opt for batch size of 3, recurrent kernel size $q = 3$, learning rate of $10^{-3}$. Our experiments are implemented in Tensorflow and run on a single Nvidia Tesla P100 GPU for each experiment. For each experiment we run 150 epochs and choose the model's state showing smallest validation loss value for testing.

\begin{table}
\caption{Tatoeba Spa-to-Eng NMT results. We ran 3 seeds for each model. Below we present the average test loss across these seeds, as well as the associated standard deviation. We did not run additional seeds for non-CWY orthogonal parameterization approaches as these methods are slow (requiring many TPU hours to train) and our primary comparison with them was w.r.t. speed.}
\label{table:NLPSTATS}
\begin{center}
\begin{tabular}{lll}
\textbf{MODEL} & \textbf{TEST CE LOSS} & \textbf{STANDARD ERROR} \\
\hline
RNN & 0.74 & .08  \\
GRU & 0.56 & .05    \\
LSTM & 0.55 & .05  \\
RGD & 2.01 & .14  \\
CWY, $L=1024$ & 0.56 & .03 \\
CWY, $L=512$ & 0.66 & .03 \\
CWY, $L=256$ & 0.64 & .06 \\
CWY, $L=128$ & 0.50 & .01 \\
CWY, $L=64$ & 0.60 & .01 \\
\end{tabular}
\end{center}
\end{table}

\begin{table}
\begin{center}
\caption{KTH action dataset statistics.}
\label{table:stats}
\begin{tabular}{lllllll}
\textbf{STATISTIC} & \textbf{WALK} & \textbf{JOG} & \textbf{RUN} & \textbf{BOX} & \textbf{WAVE} & \textbf{CLAP} \\
\hline
Min sequence length & 62 & 42 & 26 & 42 & 62 & 24 \\
Max sequence length & 231 & 152 & 111 & 362 & 245 & 235 \\
Mean sequence length & 109.3 & 68.0 & 48.9 & 110.3 & 129.0 & 106.2 \\
Total frames count (train set) & 20122 & 12730 & 9096 & 20515 & 23958 & 19529 \\
Total frames count (val. set) & 7622 & 4551 & 3448 & 7558 & 8436 & 6415 \\
Total frames count (test set) & 15991 & 9913 & 7018 & 15277 & 18963 & 16125 \\
\end{tabular}
\end{center}
\end{table}

\begin{figure}
\begin{center}
\centerline{\includegraphics[width=0.6\linewidth]{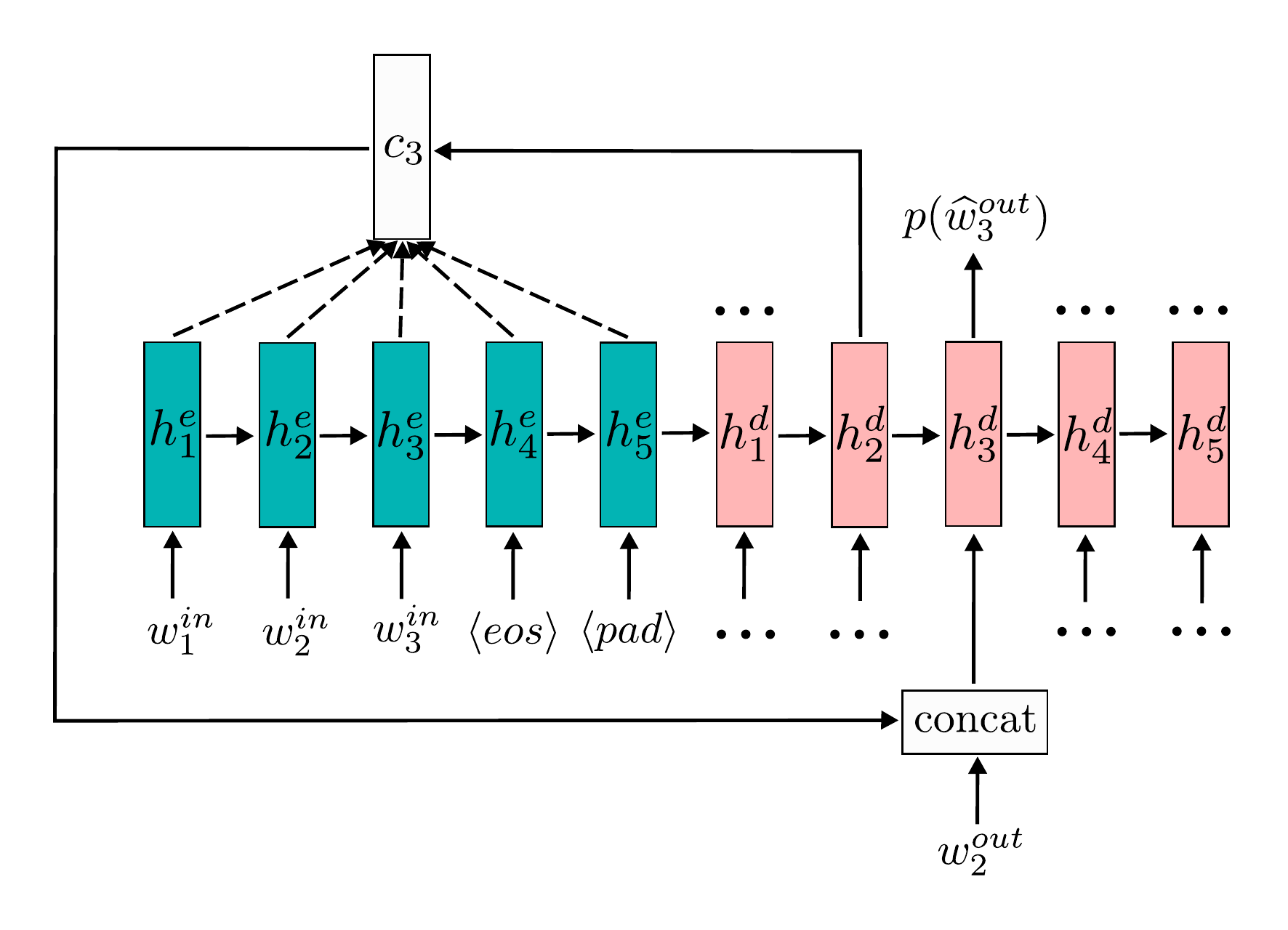}}
\caption{Sketch of the architecture used for Neural Machine Translation experiments. For  ease of illustration we use $5$ as maximal input and output length. $w_i^{in}, w_i^{out}$ are input and output word embeddings respectively, $\langle eos \rangle$ and $\langle pad \rangle$ denote embeddings of ``end of sentence" and ``padding" tag respectively. We use two different RNN units for the encoder rollout $h_1^e \to \dots \to h_5^e$ (blue) and decoder rollout $h_1^d \to \dots \to h_5^d$ (pink). We illustrate how the distribution of predicted output word $\widehat{w}_3^{out}$ is computed, other output words are processed similarly. Given $h_2^d$, the context vector $c_3 \in \mathbb{R}^N$ is computed as $\sum_i \alpha_i h_i^e$ where $\sum_i \alpha_i = 1$, $\alpha_i \propto \exp (v^\top \mathrm{tanh} (W_1 h_i^e + W_2 h_2^d))$, $v \in \mathbb{R}^N, W_1, W_2 \in \mathbb{R}^{N \times N}$ are learnable parameters. Then $c_3$ is concatenated with previous word embedding ($w_2^{out}$ or null tag embedding for the first predicted word) and passed into decoder RNN as input. Decoder RNN output ($h_3^d$) is passed through linear layer + softmax to obtain a distribution over $\widehat{w}_3^{out}$.}
\label{fig:att}
\end{center}
\end{figure}

\begin{figure}
\begin{center}
\centerline{\includegraphics[width=\linewidth]{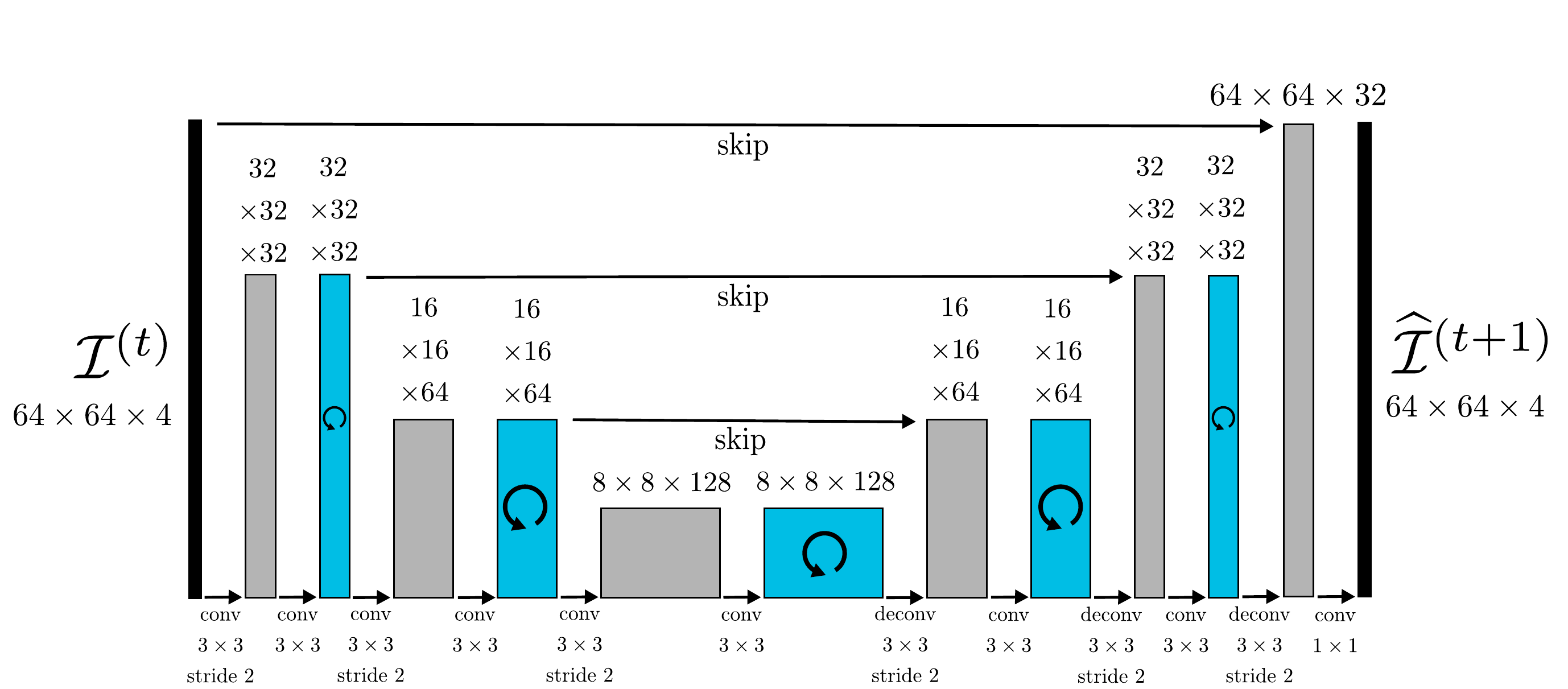}}
\caption{Sketch of the architecture used for video prediction experiments. Blue and grey blocks illustrate hidden representations with and without recurrent connections respectively. We compare different designs of the blue block (ConvLSTM, ConvNERU). In our comparison we try different designs of blue recurrent units with everything else unchanged. As in the original papers \citep{arch1,arch2}, we find that ConvLSTM version works best when instance normalization \citep{instance} is added after each convolution and before the nonlinearity, including convolutions inside ConvLSTM. We don't use instance normalization with other model variants.}
\label{fig:arch}
\end{center}
\vskip -0.2in
\end{figure}

\section{PROOFS} \label{sec:proofs}

\subsection{Theorem 1}

\begin{proof}
The proof proceeds by induction in $N$. For $N = 1$ such $Q$ is unique and is equal to $\begin{bmatrix} -1 \end{bmatrix}$. So simply take $u_1 = \begin{bmatrix} -1 \end{bmatrix}$. Now assume the statement is true for $N = k - 1 \geq 1$. When $N = k > 1$ we consider $Q$'s first column $q = \begin{bmatrix} q_1 & \dots & q_N \end{bmatrix}^\top$ and define a vector $v \in \mathbb{R}^k$ as follows:
\begin{equation} \label{eq:cases}
    v = \begin{cases}
        \frac{q - e^{(1)}}{\| q - e^{(1)}\|} & \text{if } | q_1 | < 1 \\
        \begin{bmatrix} 0 & \dots & 0 & 1 \end{bmatrix}^\top & \text{if } q_1 = 1 \\
        e^{(1)} & \text{if } q_1 = -1
    \end{cases}
\end{equation}
Observe that
\begin{equation} \label{eq:oneh}
    H(v) Q = \begin{bmatrix} 1 & r^\top \\ \mathbf{0} & Q' \end{bmatrix}
\end{equation}
for some $r \in \mathbb{R}^{k - 1}$. From the fact that $H(v) Q \in \mathcal{O}(k)$ we deduce:
\begin{equation} \label{eq:step}
    \begin{bmatrix} 1 & \mathbf{0} \\ r & Q'^\top \end{bmatrix} \begin{bmatrix} 1 & r^\top \\ \mathbf{0} & Q' \end{bmatrix} = \begin{bmatrix} 1 & r^\top \\ r & Q'^\top Q' + r r^\top \end{bmatrix} = I
\end{equation}
Hence, $r = \mathbf{0}$ and $Q' \in \mathcal{O}(k - 1)$. By Sylvester determinant identity $\det (I - 2 v v^\top) = 1 - 2 v^\top v = -1$, therefore $\det Q' = (-1)^{k - 1}$. By the induction step assumption there exist nonzero $v'^{(1)}, \dots, v'^{(k - 1)} \in \mathbb{R}^{k - 1}$ s.t.
\begin{equation*}
    Q' = H(v'^{(1)}) \dots H(v'^{(k - 1)})
\end{equation*}
We define $v^{(2)} = \begin{bmatrix} 0 & {v'^{(1)}}^\top \end{bmatrix}^\top, \dots, v^{(k)} = \begin{bmatrix} 0 & {v'^{(k - 1)}}^\top \end{bmatrix}^\top$ and obtain that
\begin{equation} \label{eq:hdecN-1}
    H (v) Q = H(v^{(2)}) \dots H (v^{(k)})
\end{equation}
Finally, we define $v^{(1)} = v$, left-multiply (\ref{eq:hdecN-1}) by $H (v^{(1)})$ and complete the induction step.
\end{proof}

\subsection{Theorem 2}

\begin{proof}
First, observe that $S$ is upper-triangular matrix with $\frac12$ on the diagonal. Hence, it is nonsingular and the Theorem statement is valid. Now the proof proceeds by induction in $L$. For $L = 1$ Theorem is trivial. Suppose Theorem is true for $L = k - 1 \geq 1$. Then the following is true:
\begin{equation*}
    H(v^{(1)}) \dots H(v^{(k - 1)}) = I - U' S'^{-1} U'^\top
\end{equation*}
where $U' = \begin{bmatrix} \frac{v^{(1)}}{\| v^{(1)} \|_2} & \dots \frac{v^{(k - 1)}}{\| v^{(k - 1)} \|_2} \end{bmatrix}$ and
\begin{equation*}
    S' = \frac{1}{2} I + \text{striu} (U'^\top U')
\end{equation*}
Then for $L = k$ we get:
\begin{align*}
    H(v^{(1)}) \dots H(v^{(k)}) &= (I - U' S'^{-1} U'^\top) H(v^{(k)}) \\
    &= I - U' S'^{-1} U'^\top - 2 \frac{v^{(k)} {v^{(k)}}^\top}{\| v^{(k)} \|_2^2} + 2 U' S'^{-1} U'^\top \frac{v^{(k)} {v^{(k)}}^\top}{\| v^{(k)} \|_2^2} \\
    &= I - U \begin{bmatrix} S'^{-1} & - 2 S'^{-1} U'^\top \frac{v^{(k)}}{\| v^{(k)} \|_2^2} \\ \mathbf{0} & 2 \end{bmatrix} U^\top
\end{align*}
And the step of induction is completed by observing that
\begin{align*}
    &\begin{bmatrix} S'^{-1} & - 2 S'^{-1} U'^\top \frac{v^{(k)}}{\| v^{(k)} \|_2^2} \\ \mathbf{0} & 2 \end{bmatrix} \times S = \begin{bmatrix} S'^{-1} & - 2 S'^{-1} U'^\top \frac{v^{(k)}}{\| v^{(k)} \|_2^2} \\ \mathbf{0} & 2 \end{bmatrix} \times \begin{bmatrix} S' & U'^\top \frac{v^{(k)}}{\| v^{(k)} \|_2^2} \\ \mathbf{0} & \frac12 \end{bmatrix} = I
\end{align*}
\end{proof}

\subsection{Theorem 3}

\begin{proof}
Similarly to Theorem 2, observe that $S$ is upper-triangular matrix with $\frac12$ on the diagonal. Hence, it is nonsingular and Theorem's statement is valid.

Observe that for any nonzero vectors $v^{(1)}, \dots v^{(M)} \in \mathbb{R}^N$
\begin{align*}
    \biggl(\begin{bmatrix} I \\ \mathbf{0} \end{bmatrix} - U S^{-1} U_1^\top \biggr)^\top \biggl( \begin{bmatrix} I \\ \mathbf{0} \end{bmatrix} - U S^{-1} U_1^\top \biggr) &= I + U_1 \biggl(S^{-\top} U^\top U S^{-1} - S^{-1} - S^{-\top} \biggr) U_1^\top \\
    &= I + U_1 S^{-\top} \biggl(U^\top U - S^\top - S \biggr) S^{-1} U_1^\top = I
\end{align*}
Hence, $\gamma_{N,M} (v^{(1)}, \dots v^{(M)}) \in \text{St} (N, M)$. To show surjectivity of $\gamma_{N,M}$, consider arbitarary $\Omega \in \text{St}(N, M)$.
Let $q = \begin{bmatrix} q_1 & \dots & q_N \end{bmatrix}^\top$ be $\Omega$'s first column. We consider value $v$ defined by (\ref{eq:cases}). Using derivations similar to (\ref{eq:oneh}-\ref{eq:step}), we obtain:
\begin{equation*}
    H(v) \Omega = \begin{bmatrix} 1 & \mathbf{0} \\ \mathbf{0} & \Omega' \end{bmatrix}
\end{equation*}
where $\Omega' \in \text{St} (N - 1, M - 1)$.

Set $v^{(1)} = v$. Analogously find $v'$ for $\Omega'$ such that
\begin{equation*}
    H(v') \Omega' = \begin{bmatrix} 1 & \mathbf{0} \\ \mathbf{0} & \Omega'' \end{bmatrix}
\end{equation*}
and set $v^{(2)} = \begin{bmatrix} 0 & v'^\top \end{bmatrix}^\top$. Repeat this procedure $M - 2$ more times to obtain:
\begin{equation} \label{eq:dec1}
    H (v^{(M)}) \dots H(v^{(1)}) \Omega = \begin{bmatrix} I \\ \mathbf{0} \end{bmatrix}
\end{equation}
Left-multiply (\ref{eq:dec1}) by $H (v^{(1)}) \dots H(v^{(M)})$:
\begin{equation*}
    \Omega = H (v^{(1)}) \dots H(v^{(M)}) \begin{bmatrix} I \\ \mathbf{0} \end{bmatrix}
\end{equation*}
Finally, apply Theorem 2 for series of Householder reflections $H (v^{(1)}) \dots H(v^{(M)})$:
\begin{align*}
    \Omega = \biggl( I - U S^{-1} U^\top \biggr) \begin{bmatrix} I \\ \mathbf{0} \end{bmatrix} = \begin{bmatrix} I \\ \mathbf{0} \end{bmatrix} - U S^{-1} U_1^\top = \gamma_{N,M} (v^{(1)}, \dots, v^{(M)})
\end{align*}
which justifies surjectivity of $\gamma_{N,M}$.
\end{proof}

\subsection{Theorem \ref{th:conv}} \label{sec:th4pr}

Before providing results which build to the complete proof, we first give a \textbf{high-level sketch} to aid intuition. Lemma \ref{lemma:def} shows that, for any iteration of SGD, $v^{(1)}, \dots, v^{(L)}$ stay in a region $\mathcal{S} = \{ x \in \mathbb{R}^N \, | \, \| x \|_2 > A \}$, where $A > 0$ is some fixed number. Lemma \ref{lemma:gl} shows that the composition of $f$ and CWY has Lipschitz-continuous gradients in $\mathcal{S}$. Next, Lemma \ref{lemma:varbnd} shows that the gradient proxy has bounded variance in $\mathcal{S}$. The proof itself is essentially Theorem 4.10 from (Bottou et al., 2016), which uses Lipschitz continuity and boundedness to establish SGD convergence guarantees.


\begin{lemma} \label{lemma:def}
Suppose conditions of Theorem \ref{th:conv} hold. Since all $v^{(0,1)}, \dots, v^{(0,L)}$ are nonzero, there exists a number $A > 0$ such that for all $l \in \{ 1, \dots, L \}: A < \| v^{(0,l)} \|_2$. Define a set $\mathcal{S} = \{ x \in \mathbb{R}^N \, | \, \| x \|_2 > A \}$. Then for each $k \geq 0$ $v^{(k,1)}, \dots, v^{(k,L)}$ are well-defined and lie in $\mathcal{S}$.
\end{lemma}
\begin{proof}
The statement is true for $k = 0$. Suppose it's true for $k - 1$. Since $v^{(k - 1,1)}, \dots, v^{(k - 1,L)}$ are nonzero, $v^{(k,1)}, \dots, v^{(k,L)}$ are well-defined. Fix $l \in \{ 1, \dots, L \}$. Observe that for any nonzero $v \in \mathbb{R}^N$: $H(v) = H (\frac{v}{\| v \|_2})$. Hence, $\widetilde{f} (H(v^{(k - 1,1)}) \dots H(v^{(k - 1,L)}))$ can be represented as a function $g (\frac{v^{(k - 1,l)}}{\| v^{(k - 1,l)} \|_2})$ so that
\begin{equation*}
    \nabla_{v^{(k - 1,l)}} \widetilde{f} (H(v^{(k - 1,1)}) \dots H(v^{(k - 1,L)})) = \nabla_{v^{(k - 1,l)}} g (\frac{v^{(k - 1,l)}}{\| v^{(k - 1,l)} \|_2}).
\end{equation*}
Denote $s (v) = \frac{v}{\| v \|_2}$. Then
\begin{equation*}
    \nabla_v g (s(v)) = \frac{1}{\| v \|_2} (I - s(v) s(v)^\top) \nabla_s g(s(v))
\end{equation*}
and, hence, $v^\top \nabla_v g (\frac{v}{\| v \|_2}) = 0$. We use it to derive that for any $\eta \in \mathbb{R}$
\begin{align}
    \| v^{(k - 1,l)} &- \eta \nabla_{v^{(k - 1,l)}} g (\frac{v^{(k - 1,l)}}{\| v^{(k - 1,l)} \|_2}) \|_2^2 = \| v^{(k - 1,l)} \|_2^2 + \| \eta \nabla_{v^{(k - 1,l)}} g (\frac{v^{(k - 1,l)}}{\| v^{(k - 1,l)} \|_2}) \|_2^2 - 2 \eta v^{(k - 1,l)\top} g (\frac{v^{(k - 1,l)}}{\| v^{(k - 1,l)} \|_2}) \nonumber \\
    &= \| v^{(k - 1,l)} \|_2^2 + \| \eta \nabla_{v^{(k - 1,l)}} g (\frac{v^{(k - 1,l)}}{\| v^{(k - 1,l)} \|_2}) \|_2^2 \geq \| v^{(k - 1,l)} \|_2^2 > A^2 > 0 . \label{eq:etadef}
\end{align}
In particular, by setting $\eta = k^{-0.5}$ and observing that $v^{(k,l)} = v^{(k - 1,l)} - k^{-0.5} \nabla_{v^{(k - 1,l)}} g (\frac{v^{(k - 1,l)}}{\| v^{(k - 1,l)} \|_2})$ we conclude that $\| v^{(k,l)} \|_2 > A$ so the step of induction is completed.
\end{proof}

\begin{lemma} \label{lemma:gl}
Suppose conditions of  Theorem \ref{th:conv} (and, hence, of Lemma \ref{lemma:def}) hold. Fix $k > 0$. According to (\ref{eq:etadef}) we can define a function $h: \mathbb{R} \to \mathbb{R}$ as
\begin{align*}
    h(\eta) &= f \biggl( H (v^{(1)} (\eta)) \dots H(v^{(L)} (\eta)) \biggr), \quad \forall l \in \{ 1, \dots, L \} : \\
    v^{(l)} (\eta) &= v^{(k - 1,l)} - \eta \cdot \nabla_{v^{(k - 1,l)}} \widetilde{f} \biggl( H(v^{(k - 1,1)}) \dots H(v^{(k - 1,L)}) \biggr) .
\end{align*}
Then
\begin{equation*}
    | \nabla h(\eta) - \nabla h(0) | \leq \mathcal{C} \widetilde{M} \eta,
\end{equation*}
where
\begin{gather*}
    \mathcal{C} = \frac{2 L}{A^2} \biggl( 5 \sqrt{6 M_2 (N + 2)} + \sqrt{2 M_2  + 48 M_1^2 (N + 2)} (\sqrt{2 (N + 60)} + 8 \sqrt{6 N (N + 2)}) \biggr), \\
    \widetilde{M} = \sum_{l = 1}^L \| \nabla_{v^{(k - 1,l)}} \widetilde{f} (H(v^{(k - 1,1)}) \dots H(v^{(k - 1,L)})) \|_2^2 .
\end{gather*}
\end{lemma}

\begin{proof}
Observe that due to (\ref{eq:etadef}) $\| v^{(l)} (\eta) \|_2 > A$. By applying a chain rule we deduce that
\begin{align*}
    \nabla h(\eta) = - \sum_{l = 1}^L \nabla_{v^{(k - 1,l)}} \widetilde{f} (H(v^{(k - 1,1)}) \dots H(v^{(k - 1,L)}))^\top \nabla_{v^{(l)}} f (H(v^{(1)} (\eta)) \dots H(v^{(L)} (\eta))) .
\end{align*}
Next, we derive that
\begin{align}
    &| \nabla h(\eta) - \nabla h(0) | = | \sum_{l = 1}^L \nabla_{v^{(k - 1,l)}} \widetilde{f} (H(v^{(k - 1,1)}) \dots H(v^{(k - 1,L)}))^\top \nonumber \\
    &\times \biggl( \nabla_{v^{(l)}} f (H(v^{(1)} (\eta)) \dots H(v^{(L)} (\eta))) - \nabla_{v^{(k - 1,l)}} f (H(v^{(k - 1,1)}) \dots H(v^{(k - 1,L)})) \biggr) | \nonumber \\
    &\leq \sqrt{\widetilde{M}} \cdot \sqrt{\sum_{l = 1}^L \| \nabla_{v^{(l)}} f (H(v^{(1)} (\eta)) \dots H(v^{(L)} (\eta))) - \nabla_{v^{(k - 1,l)}} f (H(v^{(k - 1,1)}) \dots H(v^{(k - 1,L)})) \|_2^2} \label{eq:bnd1}
\end{align}
where we use the Cauchy-Schwarz inequality. Fix $l \in \{ 1, \dots, L \}$ and let $g' (v^{(l)})$, $g'' (v^{(k - 1,l)})$ be $f (H(v^{(1)} (\eta)) \dots H(v^{(L)} (\eta)))$ and $f (H(v^{(k - 1,1)}) \dots H(v^{(k - 1,L)}))$ represented as functions of $v^{(l)}$ and $v^{(k - 1,l)}$ respectively. Then
\begin{gather*}
    \nabla_{v^{(l)}} f (H(v^{(1)} (\eta)) \dots H(v^{(L)} (\eta))) = \nabla_{v^{(l)}} g' (v^{(l)}) = \frac{1}{\| v^{(l)} \|_2} (I - s(v^{(l)}) s(v^{(l)})^\top) \nabla_s g'(s(v^{(l)})), \\
    \nabla_{v^{(k - 1,l)}} f (H(v^{(k - 1,1)}) \dots H(v^{(k - 1,L)})) = \nabla_{v^{(k - 1,l)}} g'' (v^{(k - 1,l)}) \\
    = \frac{1}{\| v^{(k - 1,l)} \|_2} (I - s(v^{(k - 1,l)}) s(v^{(k - 1,l)})^\top) \nabla_s g''(s(v^{(k - 1,l)})).
\end{gather*}
While $l$ is fixed denote $v' = v^{(l)}$ and $v'' = v^{(k - 1,l)}$. Then we have:
\begin{align}
    &\| \nabla_{v^{(l)}} g' (v^{(l)}) - \nabla_{v^{(k - 1,l)}} g'' (v^{(k - 1,l)}) \|_2 = \| \frac{1}{\| v' \|_2} (I - s(v') s(v')^\top) \nabla_s g'(s(v')) \nonumber \\
    &- \frac{1}{\| v'' \|_2} (I - s(v'') s(v'')^\top) \nabla_s g''(s(v'')) \|_2 \nonumber \\
    &= \| \frac{1}{\| v' \|_2} (I - s(v') s(v')^\top) \nabla_s g'(s(v')) - \frac{1}{\| v'' \|_2} (I - s(v') s(v')^\top) \nabla_s g'(s(v')) \nonumber \\
    &+ \frac{1}{\| v'' \|_2} (I - s(v') s(v')^\top) \nabla_s g'(s(v')) - \frac{1}{\| v'' \|_2} (I - s(v'') s(v'')^\top) \nabla_s g''(s(v'')) \|_2 \nonumber \\
    &\leq \| \frac{1}{\| v' \|_2} (I - s(v') s(v')^\top) \nabla_s g'(s(v')) - \frac{1}{\| v'' \|_2} (I - s(v') s(v')^\top) \nabla_s g'(s(v')) \|_2 \nonumber \\
    &+ \| \frac{1}{\| v'' \|_2} (I - s(v') s(v')^\top) \nabla_s g'(s(v')) - \frac{1}{\| v'' \|_2} (I - s(v'') s(v'')^\top) \nabla_s g''(s(v'')) \|_2 \nonumber \\
    &\leq | \frac{1}{\| v' \|_2} - \frac{1}{\| v'' \|_2} | \| (I - s(v') s(v')^\top) \nabla_s g'(s(v')) \|_2 \nonumber \\
    &+ \frac{1}{\| v'' \|_2} \| (I - s(v') s(v')^\top) \nabla_s g'(s(v')) - (I - s(v'') s(v'')^\top) \nabla_s g''(s(v'')) \|_2 \nonumber \\
    &\leq \frac{1}{A^2} \| v' - v'' \|_2 \| \nabla_s g'(s(v')) \|_2 + \frac{1}{A} \| (I - s(v') s(v')^\top) \nabla_s g'(s(v')) - (I - s(v'') s(v'')^\top) \nabla_s g''(s(v'')) \|_2 \nonumber \\
    &= \frac{1}{A^2} \| v' - v'' \|_2 \| \nabla_s g'(s(v')) \|_2 + \frac{1}{A} \| (I - s(v') s(v')^\top) \nabla_s g'(s(v')) \nonumber \\
    &- (I - s(v') s(v')^\top) \nabla_s g''(s(v'')) + (I - s(v') s(v')^\top) \nabla_s g''(s(v'')) - (I - s(v'') s(v'')^\top) \nabla_s g''(s(v'')) \|_2 \nonumber \\
    &\leq \frac{1}{A^2} \| v' - v'' \|_2 \| \nabla_s g'(s(v')) \|_2 + \frac{1}{A} \| (I - s(v') s(v')^\top) \biggl( \nabla_s g'(s(v')) -  \nabla_s g''(s(v'')) \biggr) \|_2 \nonumber \\
    &+ \frac{1}{A} \| \biggl( (I - s(v') s(v')^\top) - (I - s(v'') s(v'')^\top) \biggr) \nabla_s g''(s(v'')) \|_2 \nonumber \\
    &\leq \frac{1}{A^2} \| v' - v'' \|_2 \| \nabla_s g'(s(v')) \|_2 + \frac{1}{A} \| \nabla_s g'(s(v')) -  \nabla_s g''(s(v'')) \|_2 \nonumber \\
    &+ \frac{1}{A} \| (s(v') s(v')^\top - s(v'') s(v'')^\top) \nabla_s g''(s(v'')) \|_2 \nonumber \\
    &\leq \frac{1}{A^2} \| v' - v'' \|_2 \| \nabla_s g'(s(v')) \|_2 + \frac{1}{A} \| \nabla_s g'(s(v')) -  \nabla_s g''(s(v'')) \|_2 \nonumber \\
    &+ \frac{1}{A} \| s(v') s(v')^\top - s(v'') s(v'')^\top \|_2 \| \nabla_s g''(s(v'')) \|_2 \nonumber \\
    &\leq \frac{1}{A^2} \| v' - v'' \|_2 \| \nabla_s g'(s(v')) \|_2 + \frac{1}{A} \| \nabla_s g'(s(v')) -  \nabla_s g''(s(v'')) \|_2 \nonumber \\
    &+ \frac{1}{A} \| s(v') s(v')^\top - s(v') s(v'')^\top + s(v') s(v'')^\top - s(v'') s(v'')^\top \|_2 \| \nabla_s g''(s(v'')) \|_2 \nonumber \\
    &\leq \frac{1}{A^2} \| v' - v'' \|_2 \| \nabla_s g'(s(v')) \|_2 + \frac{1}{A} \| \nabla_s g'(s(v')) -  \nabla_s g''(s(v'')) \|_2 \nonumber \\
    &+ \frac{1}{A} \| s(v') ( s(v') - s(v'') )^\top \|_2 \| \nabla_s g''(s(v'')) \|_2 + \frac{1}{A} \| ( s(v') - s(v'') ) s(v'')^\top \|_2 \| \nabla_s g''(s(v'')) \|_2 \nonumber \\
    &\leq \frac{1}{A^2} \| v' - v'' \|_2 \| \nabla_s g'(s(v')) \|_2 + \frac{1}{A} \| \nabla_s g'(s(v')) -  \nabla_s g''(s(v'')) \|_2 \nonumber \\
    &+ \frac{1}{A} \| s(v') \|_2 \| s(v') - s(v'') \|_2 \| \nabla_s g''(s(v'')) \|_2 + \frac{1}{A} \| ( s(v') - s(v'') \|_2 \| s(v'') \|_2 \| \nabla_s g''(s(v'')) \|_2 \nonumber \\
    &\leq \frac{1}{A^2} \| v' - v'' \|_2 \| \nabla_s g'(s(v')) \|_2 + \frac{1}{A} \| \nabla_s g'(s(v')) -  \nabla_s g''(s(v'')) \|_2 + \frac{2}{A} \| s(v') - s(v'') \|_2 \| \nabla_s g''(s(v'')) \|_2 \nonumber \\
    &\leq \frac{1}{A^2} \| v' - v'' \|_2 ( \| \nabla_s g'(s(v')) \|_2 + 4 \| \nabla_s g''(s(v'')) \|_2 ) + \frac{1}{A} \| \nabla_s g'(s(v')) -  \nabla_s g''(s(v'')) \|_2, \label{eq:bnd2}
\end{align}
where we use submultiplicativity of the matrix norm $\| \cdot \|_2$ and that for any $v, v', v'' \in \mathcal{S}, x \in \mathbb{R}^N$: a) $I - s(v) s(v)^\top$ is an orthogonal projection matrix and, therefore, $\| I - s(v) s(v)^\top \|_2 \leq 1$, b) $| \frac{1}{\| v' \|_2} - \frac{1}{\| v'' \|_2} | = \frac{1}{\| v' \|_2 \| v'' \|_2} | \| v' \|_2 - \| v'' \|_2 | \leq \frac{1}{A^2} \| v' - v'' \|_2$ and c)
\begin{gather*}
    \| s(v') - s(v'') \|_2 = \| \frac{v'}{\| v' \|_2} - \frac{v''}{\| v'' \|_2} \|_2 = \| \frac{v'}{\| v' \|_2} - \frac{v'}{\| v'' \|_2} + \frac{v'}{\| v'' \|_2} - \frac{v''}{\| v'' \|_2} \|_2 \\
    \leq | \frac{1}{\| v' \|_2} - \frac{1}{\| v'' \|_2} | \| v' \|_2 + \frac{1}{\| v'' \|_2} \| v' - v'' \|_2 = \frac{1}{\| v'' \|_2} | \| v' \|_2 - \| v'' \|_2 | + \frac{1}{\| v'' \|_2} \| v' - v'' \|_2 \\
    \leq \frac{2}{\| v'' \|_2} \| v' - v'' \|_2 \leq \frac{2}{A} \| v' - v'' \|_2 .
\end{gather*}
For $s \in \mathbb{R}^N, \| s \|_2 = 1$ let $s_i$ denote $i$'th position of vector $s$, $H_{j_1,j_2}$ denote $(j_1, j_2)$'th position of matrix $H$ and $[ \cdot ]$ denote indicator. Then
\begin{gather*}
    \nabla_{s_i} H(s)_{j_1, j_2} = \nabla_{s_i} (1 - 2 \frac{s_{j_1} s_{j_2}}{\| s \|_2^2}) \\
    = -2 \frac{((s_{j_1} + s_{j_2}) [j_1 = i] [j_2 = i] + s_{j_1} [j_2 = i] [j_1 \neq i] + s_{j_2} [j_1 = i][j_2 \neq i]) \| s \|_2^2 - 2 s_{j_1} s_{j_2} s_i }{\| s \|_2^4} \\
    = 4 s_{j_1} s_{j_2} s_i - 2 (s_{j_1} [j_2 = i] + s_{j_2} [j_1 = i]) .
\end{gather*}
We further obtain that
\begin{gather}
    \| \nabla_s H(s) \|_F^2 = \sum_{1 \leq i, j_1, j_2 \leq N} ( 4 s_{j_1} s_{j_2} s_i - 2 (s_{j_1} [j_2 = i] + s_{j_2} [j_1 = i]) )^2 \nonumber \\
    \leq 3 \sum_{1 \leq i, j_1, j_2 \leq N} (16 s_{j_1}^2 s_{j_2}^2 s_i^2 + 4 s_{j_1}^2 [j_2 = i] + 4 s_{j_2} [j_1 = i]) \nonumber \\
    = 48 \sum_{j_1 = 1}^N s_{j_1}^2 \sum_{j_2 = 1}^N s_{j_2}^2 \sum_{i = 1}^N s_i^2 + 12 \sum_{1 \leq i, j_1 \leq N} s_{j_1}^2 + 12 \sum_{1 \leq i, j_2 \leq N} s_{j_2}^2 = 24 ( N + 2 ), \label{eq:ghb1} \\
    \| \nabla_{s'} H(s') - \nabla_{s''} H(s'') \|_F^2 = \sum_{1 \leq i, j_1, j_2 \leq N} ( 4 s_{j_1}' s_{j_2}' s_i' - 4 s_{j_1}'' s_{j_2}'' s_i'' - 2 (s_{j_1}' - s_{j_1}'') [j_2 = i] \nonumber \\
    - 2 (s_{j_2}' - s_{j_2}'') [j_1 = i] )^2 \leq 3 \sum_{1 \leq i, j_1, j_2 \leq N} (16 (s_{j_1}' s_{j_2}' s_i' - s_{j_1}'' s_{j_2}'' s_i'')^2 + 4 (s_{j_1}' - s_{j_1}'')^2 [j_2 = i] \nonumber \\
    + 4 (s_{j_2}' - s_{j_2}'')^2 [j_1 = i] ) \leq 48 \sum_{1 \leq i, j_1, j_2 \leq N} (s_{j_1}' s_{j_2}' s_i' - s_{j_1}' s_{j_2}'' s_i'' + s_{j_1}' s_{j_2}'' s_i'' - s_{j_1}'' s_{j_2}'' s_i'')^2 \nonumber \\
    + 4 \sum_{1 \leq i, j_1 \leq N} (s_{j_1}' - s_{j_1}'')^2 + 4 \sum_{1 \leq i, j_2 \leq N} (s_{j_2}' - s_{j_2}'')^2 \nonumber \\
    \leq 96 \sum_{1 \leq i, j_1, j_2 \leq N} ((s_{j_1}' s_{j_2}' s_i' - s_{j_1}' s_{j_2}'' s_i'')^2 + (s_{j_1}' s_{j_2}'' s_i'' - s_{j_1}'' s_{j_2}'' s_i'')^2) + 8 N \| s' - s'' \|_2^2 \nonumber \\
    = 96 \sum_{1 \leq i, j_1, j_2 \leq N} ( s_{j_1}'^2 ( s_{j_2}' s_i' - s_{j_2}'' s_i'')^2 + s_{j_2}''^2 s_i''^2 (s_{j_1}' - s_{j_1}'')^2) + 8 N \| s' - s'' \|_2^2 \nonumber \\
    = 96 \sum_{1 \leq i, j_1, j_2 \leq N} ( s_{j_1}'^2 ( s_{j_2}' s_i' - s_{j_2}' s_i'' + s_{j_2}' s_i'' - s_{j_2}'' s_i'')^2 + s_{j_2}''^2 s_i''^2 (s_{j_1}' - s_{j_1}'')^2) + 8 N \| s' - s'' \|_2^2 \nonumber \\
    = 96 \sum_{1 \leq i, j_1, j_2 \leq N} ( 2 s_{j_1}'^2 s_{j_2}'^2 ( s_i' - s_i'')^2 + 2 s_{j_1}'^2 s_i''^2 (s_{j_2}' - s_{j_2}'')^2 + s_{j_2}''^2 s_i''^2 (s_{j_1}' - s_{j_1}'')^2) \nonumber \\
    + 8 N \| s' - s'' \|_2^2 \leq 96 \cdot 5 \sum_{j_1 = 1}^N s_{j_1}'^2 \sum_{j_2 = 1}^N s_{j_2}'^2 \sum_{i = 1}^N (s_i' - s_i'')^2 + 8 N \| s' - s'' \|_2 \leq 8 (60 + N) \| s' - s'' \|_2^2 \label{eq:ghb2}
\end{gather}
where we use the Cauchy-Schwarz inequality and, in particular, that $(a + b)^2 \leq 2( a^2 + b^2 )$ and $(a + b + c)^2 \leq 3 (a^2 + b^2 + c^2)$.

By Jensen's inequality, for every $X \in \mathcal{O} (N)$ we have $\| \nabla f (X) \|_F^2 = \| \mathbb{E} \nabla \widetilde{f} (X) \|_F^2 \leq \mathbb{E} \| \nabla \widetilde{f} (X) \|_F^2 \leq M_2^2$. By $X' (s)$, $X'' (s)$ denote $H(v^{(1)} (\eta)) \dots H (s (v')) \dots H (v^{(L)} (\eta))$ and $H(v^{(k - 1,1)}) \dots H (s (v'')) \dots H (v^{(k - 1,L)})$ as functions of $s (v')$ and $s(v'')$ respectively. Then
\begin{gather}
    \| \nabla_s X' (s) \|_F^2 = \sum_{i = 1}^N \| \nabla_{s_i} X' (s) \|_F^2 = \sum_{i = 1}^N \| H(v^{(1)} (\eta)) \dots \nabla_{s_i} H (s (v')) \dots H (v^{(L)} (\eta)) \|_F^2 \nonumber \\
    = \sum_{i = 1}^N \| \nabla_{s_i} H (s (v')) \|_F^2 = \| \nabla_s H (s (v')) \|_F^2 \leq 24 (N + 2), \nonumber \\
    \| \nabla_s X' (s (v')) - \nabla_s X'' (s (v'')) \|_F = \| \nabla_s X' (s (v')) - \nabla_s X' (s (v'')) + \nabla_s X' (s (v'')) \nonumber \\
    - \nabla_s X'' (s (v'')) \|_F \leq \| \nabla_s X' (s (v')) - \nabla_s X' (s (v'')) \|_F + \| \nabla_s X' (s (v''))
    - \nabla_s X'' (s (v'')) \|_F \nonumber \\
    = \sqrt{ \sum_{i = 1}^N \| H(v^{(1)} (\eta)) \dots ( \nabla_{s_i} H (s (v')) - \nabla_{s_i} H (s (v'')) ) \dots H (v^{(L)} (\eta)) \|_F^2 } + \| \nabla_s X' (s (v''))
    \nonumber \\
    - \nabla_s X'' (s (v'')) \|_F = \| \nabla_s H (s (v')) - \nabla_s H (s (v'')) \|_F + \| \nabla_s X' (s (v''))
    - \nabla_s X'' (s (v'')) \|_F \nonumber \\
    \leq 2 \sqrt{2 (N + 60)} \| s (v') - s (v'') \|_2 + \sum_{i = 1}^N \| H (v^{(1)} (\eta)) \dots \nabla_{s_i} H (s (v'')) \dots H(v^{(L)} (\eta)) \nonumber \\
    - H (v^{(k - 1,1)}) \dots \nabla_{s_i} H (s (v'')) \dots H (v^{(k - 1,L)}) \|_F \label{eq:xdiff1}
\end{gather}

For every $1 \leq i \leq N$ we have:
\begin{gather}
    \| H (v^{(1)} (\eta)) \dots \nabla_{s_i} H (s (v'')) \dots H(v^{(L)} (\eta)) - H (v^{(k - 1,1)}) \dots \nabla_{s_i} H (s (v'')) \dots H (v^{(k - 1,L)}) \|_F \label{eq:xdiff2} \\
    = \| H (v^{(1)} (\eta)) \dots \nabla_{s_i} H (s (v'')) \dots H(v^{(L)} (\eta)) - H (v^{(k - 1, 1)}) H (v^{(2)} (\eta)) \dots \nabla_{s_i} H (s (v'')) \dots H(v^{(L)} (\eta)) \nonumber \\
    + H (v^{(k - 1, 1)}) H (v^{(2)} (\eta)) \dots \nabla_{s_i} H (s (v'')) \dots H(v^{(L)} (\eta)) - H (v^{(k - 1,1)}) \dots \nabla_{s_i} H (s (v'')) \dots H (v^{(k - 1,L)}) \|_F \nonumber \\
    \leq  \| H (v^{(1)} (\eta)) \dots \nabla_{s_i} H (s (v'')) \dots H(v^{(L)} (\eta)) - H (v^{(k - 1, 1)}) H (v^{(2)} (\eta)) \dots \nabla_{s_i} H (s (v')) \dots H(v^{(L)} (\eta)) \|_F \nonumber \\
    + \| H (v^{(k - 1, 1)}) \biggl( H (v^{(2)} (\eta)) \dots \nabla_{s_i} H (s (v'')) \dots H(v^{(L)} (\eta)) - \dots \nabla_{s_i} H (s (v'')) \dots H (v^{(k - 1,L)}) \biggr) \|_F \nonumber \\
    \leq \| H (v^{(1)} (\eta)) - H (v^{(k - 1, 1)}) \|_F \cdot \| H (v^{(2)} (\eta)) \dots \nabla_{s_i} H (s (v'')) \dots H(v^{(L)} (\eta)) \|_F \nonumber \\
    + \| H (v^{(2)} (\eta)) \dots \nabla_{s_i} H (s (v'')) \dots H(v^{(L)} (\eta)) - H (v^{(k - 1,2)}) \dots \nabla_{s_i} H (s (v'')) \dots H (v^{(k - 1,L)}) \|_F \nonumber \\
    \leq \| H (v^{(1)} (\eta)) - H (v^{(k - 1, 1)}) \|_F \| \nabla_{s_i} H (s (v'')) \|_F \nonumber \\
    + \| H (v^{(2)} (\eta)) \dots \nabla_{s_i} H (s (v'')) \dots H(v^{(L)} (\eta)) - H (v^{(k - 1,2)}) \dots \nabla_{s_i} H (s (v'')) \dots H (v^{(k - 1,L)}) \|_F \label{eq:xdiff3} \\
    \leq \dots \leq \| \nabla_{s_i} H (s (v'')) \|_F \sum_{l' = 1}^{l - 1} \| H (v^{(l')} (\eta)) - H (v^{(k - 1, l')}) \|_F \nonumber \\
    + \| \nabla_{s_i} H (s (v'')) \dots H(v^{(L)} (\eta)) - \nabla_{s_i} H (s (v'')) \dots H (v^{(k - 1,L)}) \|_F, \nonumber 
\end{gather}
where $\dots$ correspond to repeating the reduction of type (\ref{eq:xdiff2}-\ref{eq:xdiff3}) to the term
\begin{equation*}
    \| H (v^{(2)} (\eta)) \dots \nabla_{s_i} H (s (v'')) \dots H(v^{(L)} (\eta)) - H (v^{(k - 1,2)}) \dots \nabla_{s_i} H (s (v'')) \dots H (v^{(k - 1,L)}) \|_F 
\end{equation*}
and so on until it becomes
\begin{equation}
    \| \nabla_{s_i} H (s (v'')) \dots H(v^{(L)} (\eta)) - \nabla_{s_i} H (s (v'')) \dots H (v^{(k - 1,L)}) \|_F . \label{eq:xdiff4}
\end{equation}
Then, one can repeat the reduction of type (\ref{eq:xdiff2}-\ref{eq:xdiff3}) to (\ref{eq:xdiff4}), but by extracting right-hand side reflections, so that (\ref{eq:xdiff4}) becomes $\| \nabla_{s_i} H (s (v'')) - \nabla_{s_i} H (s (v'')) \|_F = 0$ and (\ref{eq:xdiff2}) is continued as
\begin{gather*}
    \| H (v^{(1)} (\eta)) \dots \nabla_{s_i} H (s (v')) \dots H(v^{(L)} (\eta)) - H (v^{(k - 1,1)}) \dots \nabla_{s_i} H (s (v'')) \dots H (v^{(k - 1,L)}) \|_F \\
    \leq \| \nabla_{s_i} H (s (v')) \|_F \sum_{l' \neq l} \| H (v^{(l')} (\eta)) - H (v^{(k - 1, l')}) \|_F + \| \nabla_{s_i} H (s (v')) - \nabla_{s_i} H (s (v'')) \|_F \\
    \leq  \| \nabla_{s_i} H (s (v')) \|_F \sum_{l' = 1}^L \| H (v^{(l')} (\eta)) - H (v^{(k - 1, l')}) \|_F
\end{gather*}
We sum this inequality for $1 \leq i \leq N$, apply Cauchy-Schwarz inequality and use (\ref{eq:ghb1},\ref{eq:ghb2}) to obtain that
\begin{gather}
    \sum_{i = 1}^N \| H (v^{(1)} (\eta)) \dots \nabla_{s_i} H (s (v')) \dots H(v^{(L)} (\eta)) - H (v^{(k - 1,1)}) \dots \nabla_{s_i} H (s (v'')) \dots H (v^{(k - 1,L)}) \|_F \nonumber \\
    \leq \sqrt{N} \sqrt{\sum_{i = 1}^N \| \nabla_{s_i} H (s (v')) \|_F^2} \sum_{l' = 1}^L \| H (v^{(l')} (\eta)) - H (v^{(k - 1, l')}) \|_F  \\
    = 2 \sqrt{N} \| \nabla_s H (s (v')) \|_F \sum_{l' = 1}^L \| \frac{v^{(l')} (\eta) v^{(l')} (\eta)^\top}{\| v^{(l')} (\eta) \|_2^2} - \frac{v^{(k - 1, l')} v^{(k - 1, l') \top}}{\| v^{(k - 1, l')} \|_2^2} \|_F \\
    \leq 4 \sqrt{6 N (N + 2)} \sum_{l' = 1}^L \| \frac{v^{(l')} (\eta) v^{(l')} (\eta)^\top}{\| v^{(l')} (\eta) \|_2^2} - \frac{v^{(k - 1, l')} v^{(k - 1, l') \top}}{\| v^{(k - 1, l')} \|_2^2} \|_F \label{eq:xdiff5}
\end{gather}
For each $1 \leq l' \leq L$ we have
\begin{gather*}
    \| \frac{v^{(l')} (\eta) v^{(l')} (\eta)^\top}{\| v^{(l')} (\eta) \|_2^2} - \frac{v^{(k - 1, l')} v^{(k - 1, l') \top}}{\| v^{(k - 1, l')} \|_2^2} \|_F \\
    \leq  \| \frac{v^{(l')} (\eta) v^{(l')} (\eta)^\top}{\| v^{(l')} (\eta) \|_2^2} - \frac{v^{(l')} (\eta) v^{(l')} (\eta)^\top}{\| v^{(l')} (\eta) \|_2 \| v^{(k - 1, l')} \|_2} + \frac{v^{(l')} (\eta) v^{(l')} (\eta)^\top}{\| v^{(l')} (\eta) \|_2 \| v^{(k - 1, l')} \|_2} - \frac{v^{(k - 1, l')} v^{(k - 1, l') \top}}{\| v^{(k - 1, l')} \|_2^2} \|_F \\
    \leq | \frac{1}{\| v^{(l')} (\eta) \|_2} - \frac{1}{\| v^{(k - 1, l')} \|_2} | \cdot \frac{1}{\| v^{(l')} (\eta) \|_2} \| v^{(l')} (\eta) v^{(l')} (\eta)^\top \|_F \\
    + \frac{1}{\| v^{(k - 1, l')} \|_2} \| \frac{v^{(l')} (\eta) v^{(l')} (\eta)^\top}{\| v^{(l')} (\eta) \|_2} - \frac{v^{(k - 1, l')} v^{(k - 1, l') \top}}{\| v^{(k - 1, l')} \|_2} \|_F \\
    \leq | \| v^{(l')} (\eta) \|_2 - \| v^{(k - 1, l')} \|_2 | \\
    \cdot \frac{1}{\| v^{(l')} (\eta) \|_2^2 \| v^{(k - 1, l')} \|_2} \| v^{(l')} (\eta) \|_2^2 + \frac{1}{\| v^{(k - 1, l')} \|_2} \| \frac{v^{(l')} (\eta) v^{(l')} (\eta)^\top}{\| v^{(l')} (\eta) \|_2} - \frac{v^{(k - 1, l')} v^{(k - 1, l') \top}}{\| v^{(k - 1, l')} \|_2} \|_F \\
    \leq \frac{1}{A} \| v^{(l')} (\eta) - v^{(k - 1, l')} \|_2 \\
    + \frac{1}{\| v^{(k - 1, l')} \|_2} \| \frac{v^{(l')} (\eta) v^{(l')} (\eta)^\top}{\| v^{(l')} (\eta) \|_2} - \frac{v^{(k - 1, l')} v^{(k - 1, l') \top}}{\| v^{(l')} (\eta) \|_2} + \frac{v^{(k - 1, l')} v^{(k - 1, l') \top}}{\| v^{(l')} (\eta) \|_2} - \frac{v^{(k - 1, l')} v^{(k - 1, l') \top}}{\| v^{(k - 1, l')} \|_2} \|_F \\
    \leq \frac{1}{A} \| v^{(l')} (\eta) - v^{(k - 1, l')} \|_2 + \frac{1}{\| v^{(k - 1, l')} \|_2 \| v^{(l')} (\eta) \|_2} \| v^{(l')} (\eta) v^{(l')} (\eta)^\top - v^{(k - 1, l')} v^{(k - 1, l') \top} \|_F \\
    + \frac{1}{\| v^{(k - 1, l')} \|_2} | \frac{1}{\| v^{(l')} (\eta) \|_2} - \frac{1}{\| v^{(k - 1, l')} \|_2} | \| v^{(k - 1, l')} v^{(k - 1, l') \top} \|_F \\
    = \frac{1}{A} \| v^{(l')} (\eta) - v^{(k - 1, l')} \|_2 + \frac{1}{\| v^{(k - 1, l')} \|_2 \| v^{(l')} (\eta) \|_2} \| v^{(l')} (\eta) v^{(l')} (\eta)^\top - v^{(k - 1, l')} v^{(k - 1, l') \top} \|_F \\
    + \frac{1}{\| v^{(k - 1, l')} \|_2^2 \| v^{(l')} (\eta) \|_2} | \| v^{(l')} (\eta) \|_2 - \| v^{(k - 1, l')} \|_2 | \| v^{(k - 1, l')} \|_2^2 \\
    \leq \frac{1}{A} \| v^{(l')} (\eta) - v^{(k - 1, l')} \|_2 \\
    + \frac{1}{\| v^{(k - 1, l')} \|_2 \| v^{(l')} (\eta) \|_2} \| v^{(l')} (\eta) v^{(l')} (\eta)^\top - v^{(k - 1, l')} v^{(k - 1, l') \top} \|_F + \frac{1}{\| v^{(l')} (\eta) \|_2} \| v^{(l')} (\eta) - v^{(k - 1, l')} \|_2 \\
    \leq \frac{2}{A} \| v^{(l')} (\eta) - v^{(k - 1, l')} \|_2 \\
    + \frac{1}{\| v^{(k - 1, l')} \|_2 \| v^{(l')} (\eta) \|_2} \| v^{(l')} (\eta) v^{(l')} (\eta)^\top - v^{(l')} (\eta) v^{(k - 1, l') \top} + v^{(l')} (\eta) v^{(k - 1, l') \top} - v^{(k - 1, l')} v^{(k - 1, l') \top} \|_F \\
    \leq \frac{2}{A} \| v^{(l')} (\eta) - v^{(k - 1, l')} \|_2 \\
    + \frac{1}{\| v^{(k - 1, l')} \|_2 \| v^{(l')} (\eta) \|_2} ( \| v^{(l')} (\eta) ( v^{(l')} (\eta) - v^{(k - 1, l')} )^\top \|_F + \| ( v^{(l')} (\eta) - v^{(k - 1, l')} ) v^{(k - 1, l') \top} \|_F ) \\
    \leq \frac{2}{A} \| v^{(l')} (\eta) - v^{(k - 1, l')} \|_2 + \frac{1}{\| v^{(k - 1, l')} \|_2 \| v^{(l')} (\eta) \|_2} (\| v^{(l')} (\eta) \|_2 + \| v^{(k - 1, l')} \|_2) \| v^{(l')} (\eta) - v^{(k - 1, l')} \|_2 \\
    \leq \frac{2}{A} \| v^{(l')} (\eta) - v^{(k - 1, l')} \|_2 + (\frac{1}{\| v^{(l')} (\eta) \|_2} + \frac{1}{\| v^{(k - 1, l')} \|_2}) \| v^{(l')} (\eta) - v^{(k - 1, l')} \|_2 \leq \frac{4}{A} \| v^{(l')} (\eta) - v^{(k - 1, l')} \|_2
\end{gather*}
We combine this with (\ref{eq:xdiff1}, \ref{eq:xdiff5}) and conclude that
\begin{gather*}
    \| \nabla_s X' (s (v')) - \nabla_s X'' (s (v'')) \|_F \leq 2 \sqrt{2 (N + 60)} \| s (v') - s (v'') \|_2 \\
    + 4 \sqrt{6 N (N + 2)} \sum_{l' = 1}^L \biggl( \frac{4}{A} \| v^{(l')} (\eta) - v^{(k - 1, l')} \|_2 \biggr) \\
    \leq \frac{2}{A} \sqrt{2 (N + 60)} \| v^{(l)} - v^{(k - 1,l)} \|_2 + 4 \sqrt{6 N (N + 2)} \sum_{l' = 1}^L \frac{4}{A} \| v^{(l')} (\eta) - v^{(k - 1, l')} \|_2 \\
    \leq \frac{2}{A} \biggl( \sqrt{2 (N + 60)} + 8 \sqrt{6 N (N + 2)} \biggr) \sum_{l' = 1}^L \| v^{(l')} (\eta) - v^{(k - 1, l')} \|_2
\end{gather*}

Next, we deduce that
\begin{gather*}
    \| \nabla_s g'(s(v')) \|_2^2 = \sum_{i = 1}^N ( \nabla_{s_i} g'(s(v')) )^2 = \sum_{i = 1}^N ( \nabla_{s_i} f ( X' (s(v')) ) )^2 \\
    = \sum_{i = 1}^N \mathrm{Trace} ( \nabla f ( X' (s(v')) )^\top \nabla_{s_i} X' (s(v')) )^2 \leq \sum_{i = 1}^N \| \nabla f ( X' (s(v')) ) \|_F^2 \| \nabla_{s_i} X' (s(v')) \|_F^2 \\
    \leq M_2 \sum_{i = 1}^N \| \nabla_{s_i} X' (s(v')) \|_F^2 = M_2 \| \nabla_s X' (s(v')) \|_F^2 \leq 24 M_2 (N + 2).
\end{gather*}
Analogously it is derived that $\| \nabla_s g'' (s(v'')) \|_2^2 \leq 24 M_2 (N + 2)$. We proceed by observing that
\begin{gather*}
    \| \nabla_s g'(s(v')) -  \nabla_s g''(s(v'')) \|_2^2 = \sum_{i = 1}^N ( \nabla_{s_i} g'(s(v')) -  \nabla_{s_i} g''(s(v'')) )^2 \\
    = \sum_{i = 1}^N ( \mathrm{Trace} ( \nabla f ( X' (s(v')) )^\top \nabla_{s_i} X' (s(v')) ) -  \mathrm{Trace} ( \nabla f ( X'' (s(v'')) )^\top \nabla_{s_i} X'' (s(v'')) ) )^2 \\
    = \sum_{i = 1}^N ( \mathrm{Trace} ( \nabla f ( X' (s(v')) )^\top \nabla_{s_i} X' (s(v')) ) - \mathrm{Trace} ( \nabla f ( X' (s(v')) )^\top \nabla_{s_i} X'' (s(v'')) ) \\
    + \mathrm{Trace} ( \nabla f ( X' (s(v')) )^\top \nabla_{s_i} X'' (s(v'')) ) - \mathrm{Trace} ( \nabla f ( X'' (s(v'')) )^\top \nabla_{s_i} X'' (s(v'')) ) )^2 \\
    \leq 2 \sum_{i = 1}^N (\mathrm{Trace} ( \nabla f ( X' (s(v')) )^\top ( \nabla_{s_i} X' (s(v'))  - \nabla_{s_i} X'' (s(v'')) ))^2 \\
    + \mathrm{Trace} ( ( \nabla f ( X' (s(v')) ) - \nabla f ( X'' (s(v'')) ))^\top \nabla_{s_i} X'' (s(v'')) )^2) \\
    \leq 2 \sum_{i = 1}^N ( \| \nabla f ( X' (s(v')) ) \|_F^2 \| \nabla_{s_i} X' (s(v'))  - \nabla_{s_i} X'' (s(v'')) \|_F^2 \\
    + \| \nabla f ( X' (s(v')) ) - \nabla f ( X'' (s(v'')) ) \|_F^2 \| \nabla_{s_i} X'' (s(v'')) \|_F^2 ) \\
    \leq 2 M_2 \sum_{i = 1}^N \| \nabla_{s_i} X' (s(v'))  - \nabla_{s_i} X'' (s(v'')) \|_F^2 + 2 M_1^2 \| X' (s(v')) - X'' (s(v'')) \|_F^2 \sum_{i = 1}^N \| \nabla_{s_i} X'' (s(v'')) \|_F^2 \\
    \leq ( 2 M_2  + 2 M_1^2 \| \nabla_s X'' (s(v'')) \|_F^2) \| \nabla_s X' (s(v'))  - \nabla_s X'' (s(v'')) \|_F^2 \\
    \leq ( 2 M_2  + 48 M_1^2 (N + 2) ) (\frac{2}{A} ( \sqrt{2 (N + 60)} + 8 \sqrt{6 N (N + 2)} ) \sum_{l' = 1}^L \| v^{(l')} (\eta) - v^{(k - 1, l')} \|_2)^2.
\end{gather*}
We continue (\ref{eq:bnd2}) and deduce that
\begin{align*}
    &\| \nabla_{v^{(l)}} g' (v^{(l)}) - \nabla_{v^{(k - 1,l)}} g'' (v^{(k - 1,l)}) \|_2 \leq \frac{10}{A^2} \sqrt{6 M_2 (N + 2)} \| v' - v'' \|_2 \\
    &+ \frac{2}{A^2} \sqrt{2 M_2  + 48 M_1^2 (N + 2)} (\sqrt{2 (N + 60)} + 8 \sqrt{6 N (N + 2)} ) \sum_{l' = 1}^L \| v^{(l')} (\eta) - v^{(k - 1, l')} \|_2 \\
    &\leq \frac{\mathcal{C}}{L} \sum_{l' = 1}^L \| v^{(l')} (\eta) - v^{(k - 1, l')} \|_2 \leq \frac{\mathcal{C}}{\sqrt{L}} \sqrt{ \sum_{l' = 1}^L \| v^{(l')} (\eta) - v^{(k - 1, l')} \|_2^2} .
\end{align*}
We plug the last inequality into (\ref{eq:bnd1}) to obtain that
\begin{align*}
    &| \nabla h(\eta) - \nabla h(0) | \leq \frac{\mathcal{C}}{\sqrt{L}} \cdot \sqrt{\widetilde{M}} \cdot \sqrt{L \sum_{l' = 1}^L \| v^{(l')} (\eta) - v^{(k - 1, l')} \|_2^2} \\
    &= \mathcal{C} \sqrt{\widetilde{M}} \cdot \sqrt{\sum_{l = 1}^L \| - \eta \nabla_{v^{(k - 1,l)}} \widetilde{f} (H(v^{(k - 1,1)}) \dots H(v^{(k - 1,L)})) \|_2^2} \\
    &= \eta \cdot \mathcal{C} \sqrt{\widetilde{M}} \sqrt{\sum_{l = 1}^L \| \nabla_{v^{(k - 1,l)}} \widetilde{f} (H(v^{(k - 1,1)}) \dots H(v^{(k - 1,L)})) \|_2^2} =\mathcal{C} \widetilde{M}  \eta .
\end{align*}
\end{proof}

\begin{lemma} \label{lemma:varbnd}
Suppose conditions of Theorem \ref{th:conv} (and, consequently, of Lemma \ref{lemma:def}) hold. For any $v^{(1)}, \dots, v^{(L)} \in \mathcal{S}$
\begin{equation*}
    \mathbb{E} [ \sum_{l = 1}^L \| \nabla_{v^{(l)}} \widetilde{f} (H(v^{(1)}) \dots H(v^{(L)})) \|_2^2 ] \leq \mathcal{D}, \quad \mathcal{D} = \frac{24}{A^2} N (N + 2) L M_2 .
\end{equation*}
\end{lemma}
\begin{proof}
For each $1 \leq l \leq L$ we have
\begin{gather*}
    \| \nabla_{v^{(l)}} \widetilde{f} (H(v^{(1)}) \dots H(v^{(L)})) \|_2^2 = \sum_{i = 1}^N ( \nabla_{v_i^{(l)}} \widetilde{f} (H(v^{(1)}) \dots H(v^{(L)})) )^2 \\
    = \sum_{i = 1}^N \mathrm{Trace} ( \nabla  \widetilde{f} (H(v^{(1)}) \dots H(v^{(L)}))^\top \nabla_{v_i^{(l)}} \biggl( H(v^{(1)}) \dots H(v^{(L)}) \biggr) )^2 \\
    \leq \sum_{i = 1}^N \| \nabla \widetilde{f} (H(v^{(1)}) \dots H(v^{(L)})) \|_F^2 \cdot \| H(v^{(1)}) \dots \nabla_{v_i^{(l)}} H(v^{(l)}) \dots H(v^{(L)}) \|_F^2 \\
    = \| \nabla \widetilde{f} (H(v^{(1)}) \dots H(v^{(L)})) \|_F^2 \sum_{i = 1}^N \| \nabla_{v_i^{(l)}} H(v^{(l)}) \|_F^2 = \| \nabla \widetilde{f} (H(v^{(1)}) \dots H(v^{(L)})) \|_F^2 \| \nabla_{v^{(l)}} H(v^{(l)}) \|_F^2 \\
    = \| \nabla \widetilde{f} (H(v^{(1)}) \dots H(v^{(L)})) \|_F^2 \sum_{i = 1}^N \| \nabla_{v_i^{(l)}} H(s (v^{(l)}) ) \|_F^2 \\
    = \| \nabla \widetilde{f} (H(v^{(1)}) \dots H(v^{(L)})) \|_F^2 \sum_{i = 1}^N \| \sum_{j = 1}^N \nabla_{v_i^{(l)}} s_j(v^{(l)}) \nabla_{s_j} H(s (v^{(l)})) \|_F^2 \\
    \leq \| \nabla \widetilde{f} (H(v^{(1)}) \dots H(v^{(L)})) \|_F^2 \sum_{i = 1}^N \| \nabla_{v_i^{(l)}} s (v^{(l)}) \|_2^2 \| \nabla_s H(s (v^{(l)})) \|_F^2 \\
    = \| \nabla \widetilde{f} (H(v^{(1)}) \dots H(v^{(L)})) \|_F^2 \| \nabla_{v^{(l)}} s (v^{(l)}) \|_F^2 \| \nabla_s H(s (v^{(l)})) \|_F^2 \\
    \leq \| \nabla \widetilde{f} (H(v^{(1)}) \dots H(v^{(L)})) \|_F^2 \frac{1}{\| v^{(l)} \|_2^2} \| I - s (v^{(l)}) s (v^{(l)})^\top \|_F^2 \cdot 24 (N + 2) \\
    \leq \frac{24}{A^2} N (N + 2) \| \nabla \widetilde{f} (H(v^{(1)}) \dots H(v^{(L)})) \|_F^2 
\end{gather*}
where we use $\| I - s (v^{(l)}) s (v^{(l)})^\top \|_F^2 \leq N$ because $I - s (v^{(l)}) s (v^{(l)})^\top$ is an orthogonal projection matrix. Next, we obtain that
\begin{gather*}
    \mathbb{E} \sum_{l = 1}^L \| \nabla_{v^{(k - 1,l)}} \widetilde{f} (H(v^{(k - 1,1)}) \dots H(v^{(k - 1,L)})) \|_2^2 \leq \frac{24}{A^2} N (N + 2) \cdot \mathbb{E} L \| \nabla \widetilde{f} (H(v^{(k - 1,1)}) \dots H(v^{(k - 1,L)})) \|_F^2 \leq \mathcal{D} .
\end{gather*}
\end{proof}

\begin{proof}[Theorem \ref{th:conv} proof]

As shown by Lemma \ref{lemma:def}, all step sizes are well-defined. We adapt the proof of Theorem 4.10 from \citep{bottou}. We consider a step $k$ and deduce from Lemma \ref{lemma:gl} that
\begin{gather*}
    h (k^{-0.5}) - h (0) - k^{-0.5} \nabla h(0) = \int_0^{k^{-0.5}} (\nabla h(\eta) - \nabla h (0)) d \eta \leq \int_0^{k^{-0.5}} | \nabla h(\eta) - \nabla h (0) | d \eta \\
    \leq \mathcal{C} \widetilde{M} \int_0^{k^{-0.5}} \eta d \eta = \frac{\mathcal{C} \widetilde{M} k^{-1}}{2}.
\end{gather*}
By expanding $h$'s definition, we deduce
\begin{gather*}
    f \biggl( H (v^{(k,1)}) \dots H(v^{(k,L)}) \biggr) - f \biggl( H (v^{(k - 1,1)}) \dots H(v^{(k - 1,L)})  \biggr) \leq \frac{\mathcal{C} \widetilde{M} k^{-1}}{2} \\
    - k^{-0.5} \sum_{l = 1}^L \nabla_{v^{(k - 1,l)}} \widetilde{f} (H(v^{(k - 1,1)}) \dots H(v^{(k - 1,L)}))^\top \nabla_{v^{(k - l,l)}} f (H(v^{(k - 1,1)}) \dots H(v^{(k - 1,L)})).
\end{gather*}
Take expectation conditioned on $\mathcal{F}_k$ -- a $\sigma$-algebra associated with $\{ \{ v^{(k',1)}, \dots, v^{(k',L)} \} \}_{k' = 1}^{k - 1}$:
\begin{gather*}
    \mathbb{E} [ f \biggl( H (v^{(k,1)}) \dots H(v^{(k,L)}) \biggr) | \mathcal{F}_k ] - f \biggl( H (v^{(k - 1,1)}) \dots H(v^{(k - 1,L)})  \biggr) \leq \frac{\mathcal{C} \mathbb{E} [ \widetilde{M} | \mathcal{F}_k ] k^{-1}}{2} \\
    - k^{-0.5} \sum_{l = 1}^L \mathbb{E} [\nabla_{v^{(k - 1,l)}} \widetilde{f} (H(v^{(k - 1,1)}) \dots H(v^{(k - 1,L)})) | \mathcal{F}_k ]^\top \times \nabla_{v^{(k - 1,l)}} f (H(v^{(k - 1,1)}) \dots H(v^{(k - 1,L)})).
\end{gather*}
By $\widetilde{f}$'s definition we have
\begin{gather*}
    \mathbb{E} [\nabla_{v^{(k - 1,l)}} \widetilde{f} (H(v^{(k - 1,1)}) \dots H(v^{(k - 1,L)})) | \mathcal{F}_k ] = \nabla_{v^{(k - 1,l)}} f (H(v^{(k - 1,1)}) \dots H(v^{(k - 1,L)}))
\end{gather*}
and, therefore,
\begin{gather}
    \mathbb{E} [ f \biggl( H (v^{(k,1)}) \dots H(v^{(k,L)}) \biggr) | \mathcal{F}_k ] - f \biggl( H (v^{(k - 1,1)}) \dots H(v^{(k - 1,L)})  \biggr) \leq \frac{\mathcal{C} \mathbb{E} [ \widetilde{M} | \mathcal{F}_k ] k^{-1}}{2} \nonumber \\
    - k^{-0.5} \sum_{l = 1}^L \| \nabla_{v^{(k - 1,l)}} f (H(v^{(k - 1,1)}) \dots H(v^{(k - 1,L)})) \|_2^2. \label{eq:bnd3}
\end{gather}

Next, we combine (\ref{eq:bnd3}) and Lemma \ref{lemma:varbnd}, applied to $\widetilde{M}$, to obtain that
\begin{gather*}
    \mathbb{E} [ f \biggl( H (v^{(k,1)}) \dots H(v^{(k,L)}) \biggr) | \mathcal{F}_k ] - f \biggl( H (v^{(k - 1,1)}) \dots H(v^{(k - 1,L)})  \biggr) \leq \frac{\mathcal{C} \mathcal{D} k^{-1}}{2} \\
    - k^{-0.5} \sum_{l = 1}^L \| \nabla_{v^{(k - 1,l)}} f (H(v^{(k - 1,1)}) \dots H(v^{(k - 1,L)})) \|_2^2.
\end{gather*}
Take full expectation and regroup:
\begin{gather*}
    k^{-0.5} \mathbb{E} \sum_{l = 1}^L \| \nabla_{v^{(k - 1,l)}} f (H(v^{(k - 1,1)}) \dots H(v^{(k - 1,L)})) \|_2^2 \leq \mathbb{E} f \biggl( H (v^{(k - 1,1)}) \dots H(v^{(k - 1,L)}) \biggr) \\
    - \mathbb{E} f \biggl( H (v^{(k,1)}) \dots H(v^{(k,L)}) \biggr) + \frac{\mathcal{C} \mathcal{D} k^{-1}}{2} .
\end{gather*}
For $K > 0$ take a sum for $1 \leq k \leq K$:
\begin{gather*}
    \sum_{k' = 1}^K k'^{-0.5} \mathbb{E} \sum_{l = 1}^L \| \nabla_{v^{(k' - 1,l)}} f (H(v^{(k' - 1,1)}) \dots H(v^{(k' - 1,L)})) \|_2^2 \leq f \biggl( H (v^{(0,1)}) \dots H(v^{(0,L)}) \biggr) \\
    - \mathbb{E} f \biggl( H (v^{(K,1)}) \dots H(v^{(K,L)}) \biggr) + \sum_{k' = 1}^K \frac{\mathcal{C} \mathcal{D} k'^{-1}}{2} .
\end{gather*}
$f$ is continuous on a compact domain $\mathcal{O} (N)$, hence there exists a minimal value $f^*$ of $f$ on $\mathcal{O} (N)$. We continue and derive that
\begin{gather*}
    \sum_{k' = 1}^K k'^{-0.5} \mathbb{E} \sum_{l = 1}^L \| \nabla_{v^{(k' - 1,l)}} f (H(v^{(k' - 1,1)}) \dots H(v^{(k' - 1,L)})) \|_2^2 \leq f \biggl( H (v^{(0,1)}) \dots H(v^{(0,L)}) \biggr) - f^* + \sum_{k' = 1}^K \frac{\mathcal{C} \mathcal{D} k'^{-1}}{2}, \\
    \min_{0 \leq k' < K} \mathbb{E} \sum_{l = 1}^L \| \nabla_{v^{(k' - 1,l)}} f (H(v^{(k' - 1,1)}) \dots H(v^{(k' - 1,L)})) \|_2^2 \\
    \leq \frac{1}{\sum_{k' = 1}^K k'^{-0.5}} \sum_{k' = 1}^K k'^{-0.5} \mathbb{E} \sum_{l = 1}^L \| \nabla_{v^{(k' - 1,l)}} f (H(v^{(k' - 1,1)}) \dots H(v^{(k' - 1,L)})) \|_2^2 \\
    \leq \frac{1}{\sum_{k' = 1}^K k'^{-0.5}} (f \biggl( H (v^{(0,1)}) \dots H(v^{(0,L)}) \biggr) - f^*) + \frac{\mathcal{C} \mathcal{D}}{2} \frac{\sum_{k' = 1}^K k'^{-1}}{\sum_{k' = 1}^K k'^{-0.5}}.
\end{gather*}
The proof is concluded by observing that $\sum_{k' = 1}^K k'^{-0.5} = \Omega (K^{0.5})$ and $\sum_{k' = 1}^K k'^{-1} = O (\log K) = o (K^\epsilon)$ for any $\epsilon > 0$.
\end{proof}

\end{document}